\documentclass{article}

\usepackage{microtype}
\usepackage{subfigure}
\usepackage{booktabs} 
\usepackage{amsmath,amsfonts,amssymb,graphicx,amsthm,enumerate,url,bm,bbm,graphicx}
\usepackage{stmaryrd}
\usepackage{xifthen,xcolor,tikz,setspace,pgfplots}
\usetikzlibrary{decorations.pathmorphing,patterns,shapes,calc}
\usepackage{mathrsfs}
\usepackage{algorithm}
\usepackage[algo2e,boxed,vlined,linesnumbered,algoruled]{algorithm2e}
\SetKwRepeat{Do}{do}{while}%
\usepackage[noend]{algorithmic}
\usepackage[colorinlistoftodos]{todonotes}
\usepackage{enumitem}
\usepackage{pgfplots}
\usepackage{times}
\pgfplotsset{compat=newest}
\usepackage{xcolor}
\usepackage{filecontents}
\usepackage{caption}
\usetikzlibrary{external}
\usepackage{hyperref}

\usepackage[accepted]{icml2020}

\newtheorem{problem}{Problem}
\newtheorem{lemma}{Lemma}
\newtheorem{theorem}{Theorem}[section]

\newtheorem{assumption}{Assumption}
\newtheorem{proposition}{Proposition}
\newtheorem{corollary}{Corollary}

\newcommand{\model}{\textsc{\small synchronization}}
\newcommand{\modelb}{\model\ \textsc{\small bandit}}
\newcommand{\modelbs}{\model\ \textsc{\small bandits}}
\newcommand{\Model}{\textsc{\small Synchronization}}

\newcommand{\Modelbs}{\Model\ \textsc{\small bandits}}
\newcommand{\algo}{\textsc{\small MirrorSync}}
\newcommand{\algopsgd}{\textsc{\small PSGDSync}}
\newcommand{\algoprac}{\textsc{\small Async}\algo}
\newcommand{\algopracpsgd}{\textsc{\small Async}\algopsgd}
\newcommand{\algof}{\textsc{\scriptsize MirrorSync}}
\newcommand{\algopsgdf}{\textsc{\scriptsize PSGDSync}}
\newcommand{\algopracf}{\textsc{\scriptsize Async}\algof}
\newcommand{\algopracpsgdf}{\textsc{\scriptsize Async}\algopsgdf}

\newcommand{\E}{\mathbb{E}}

\DeclareMathOperator*{\argmin}{\operatornamewithlimits{arg\,min}}
\DeclareMathOperator*{\argmax}{\operatornamewithlimits{arg\,max}}

\renewcommand{\comment}[1]{}
\makeatletter
\newcommand{\shorteq}{%
  \settowidth{\@tempdima}{-}
  \resizebox{\@tempdima}{\height}{=}%
}
\makeatother

\newcommand{\Reg}{\operatorname{Reg}}
\renewcommand{\epsilon}{\varepsilon}

\definecolor{rred}{HTML}{C0504D}

\newcounter{marknumber}
\pgfplotsset{
    error bars/every nth mark/.style={
        /pgfplots/error bars/draw error bar/.prefix code={
            \pgfmathtruncatemacro\marknumbercheck{mod(floor(\themarknumber/2),#1)}
            \ifnum\marknumbercheck=0
            \else
                \begin{scope}[opacity=0]
            \fi
        },
        /pgfplots/error bars/draw error bar/.append code={
            \ifnum\marknumbercheck=0
            \else
                \end{scope}
            \fi
            \stepcounter{marknumber}    
        }
    }
}

\setlist[itemize]{leftmargin=*}

\icmltitlerunning{Online Learning for Active Cache Synchronization}

\begin{document}
\setlength{\abovedisplayskip}{4pt}
\setlength{\belowdisplayskip}{4pt}
\twocolumn[
\icmltitle{Online Learning for Active Cache Synchronization}


\icmlsetsymbol{equal}{*}

\begin{icmlauthorlist}
\icmlauthor{Andrey Kolobov}{msr}
\icmlauthor{S\'ebastien Bubeck}{msr}
\icmlauthor{Julian Zimmert}{goo}
\end{icmlauthorlist}

\icmlaffiliation{msr}{Microsoft Research, Redmond}
\icmlaffiliation{goo}{Google Research, Berlin. Work on this paper partially done during a visit to Microsoft Research, Redmond}

\icmlcorrespondingauthor{Andrey Kolobov}{akolobov@microsoft.com}

\icmlkeywords{Online learning, caching, multi-armed bandits}

\vskip 0.3in
]


\printAffiliationsAndNotice{}

\begin{abstract}
Existing multi-armed bandit (MAB) models make two implicit assumptions: an arm generates a payoff only when it is played, and the agent observes every payoff that is generated. This paper introduces \emph{\modelbs}, a MAB variant where \emph{all} arms generate costs at \emph{all} times, but the agent observes an arm's instantaneous cost only when the arm is played. \Model\ MABs are inspired by online caching scenarios such as Web crawling, where an arm corresponds to a cached item and playing the arm means downloading its fresh copy from a server.
We present \algo, an online learning algorithm for \modelbs, establish an adversarial regret of $O(T^{2/3})$ for it, and show how to make it practical.
\end{abstract}
\section{Introduction}

Multi-armed bandits (MAB) \cite{robbins-ams52} have been widely applied in settings where an agent repeatedly faces $K$ choices (\emph{arms}), each associated with its own payoff distribution unknown to the agent at the start, and needs to eventually identify the arm with the highest mean payoff by pulling a subset of arms at a time and observing a payoff sampled from their distributions. MABs' defining property is that the agent observes an arm's instantaneous payoff when and only when the agent plays it. A key hidden assumption that goes hand-in-hand with it in the existing bandit models is that each arm generates reward when and only when it is played, which, combined with the bandit feedback property, also implies that the agent observes all generated payoffs. 

In this paper, we go beyond these seemingly fundamental assumptions by identifying a class of practical settings that violate them and analyzing it using online learning theory. Specifically, this paper formalizes scenarios that we call \emph{\model\ MABs}. In these settings, the agent can be thought of as holding copies of $K$ files whose originals come from different remote sources. As time goes by, the files change at the sources, and their copies increasingly differ from the originals, becoming \emph{stale}. The agent's task is to refresh these files by occasionally downloading their new copies from remote sources, under a constraint $B$ on the average number of downloads per time unit.

For each file, the agent is \emph{continually} penalized for its staleness. The expected penalty at each time step due to this file is a non-decreasing function of the time since the file's last refresh. Playing arm $k$ here corresponds to refreshing file $k$: doing so temporarily reduces its staleness and thereby diminishes the cost incurred due to it per time unit. The goal is to find a synchronization policy that minimizes regret in terms of the average staleness penalties by refreshing files according to a well-chosen schedule.

Crucially, at any moment the agent doesn't know how outdated its copy of a given file is, except at the time when it downloads its fresh copy, and therefore \emph{most of the time doesn't know the penalties it is incurring}. It observes the penalty only when it plays an arm, i.e., refreshes a file, and has a chance to see how different the cached copy was right before the refresh. Even this action reveals only the \emph{instantaneous} penalty due to this file, not the cumulative penalty the file has brought on since its last refresh.

\Model\ MABs are inspired by problems such as web crawl scheduling \cite{wolf-www02, cho-tds03, azar-pnas18, kolobov-neurips19, upadhyay-aaai20} and database update management \cite{gal-jacm01,bright-tods06}. All these settings involve a cache that must \emph{proactively} initiate downloads to refresh its content. This is in contrast to, e.g., Web browser caches that \emph{passively} monitor a stream of download requests initiated by another program. The few existing works on policy learning for active caching \cite{kolobov-neurips19,upadhyay-aaai20} apply only to specific penalty functions. In contrast, the theoretical results in this paper are independent of the penalties' functional form, and come with a practical online learning strategy for this model.  

\textbf{High-level analysis idea and paper outline.} Online learning theory is a powerful tool for analyzing decision-making models where an agent operates in discrete-time instantaneous rounds by playing a candidate solution (arm), immediately getting a feedback on it (a sample from the arm's payoff distribution), and using it to choose a candidate solution for the next round. Unfortunately, online learning's traditional assumptions clash with the properties of our setting. As Section \ref{sec:form} describes, \model\ MAB is a \emph{continuous-time} model with \emph{non-stationary} sparsely observable costs. Its candidate solutions are multi-arm \emph{policies} (Section \ref{sec:soln}). Getting useful feedback on a policy, such as an estimate of its cost function or gradient, isn't instantaneous; it requires playing the policy for a non-trivial stretch of time. 

In Section \ref{sec:algo}, we present the \algo\ algorithm, which continuously plays a candidate policy along with ``exploratory" arm pulls, periodically updating it with online mirror descent \cite{nemirovsky-83,bubeck-16}. It uses a novel unbiased policy gradient estimator that operates in the face of sparse policy cost observations. Our regret analysis of \algo\ in Section \ref{sec:analysis} critically relies on the convexity of policy cost functions -- a property we derive in Section \ref{sec:soln} from minimal assumptions on \model\ MABs' payoffs.
The regret analysis treats time intervals between \algo's policy updates as learning ``rounds" and thereby brings online learning theory to bear on \modelbs. Section \ref{sec:disc} introduces \algoprac, a practical \algo\ variant that lifts \algo's idealizing assumptions. In Section \ref{sec:exps}, we compare the two algorithm empirically.

\textbf{The contributions} of this paper are thus as follows:

     \textbf{(1)} We cast active caching as an online learning problem with sparse feedback, enabling principled theoretical analysis of this setting under a variety of payoff distributions.
    
     \textbf{(2)} Based on this formulation, we propose a theoretic strategy for active caching under unknown payoff distributions and derive an adversarial regret bound of $O(T^{2/3})$ for it. In doing so, we overcome the challenges of sparse and temporal feedback inherent in this scenario that existing online learning theory does not address.
    
     \textbf{(3)} We present a practical variant of the above strategy that lifts the latter's assumptions and, as experiments demonstrate, has the same empirical convergence rate.

\section{Model formalization \label{sec:form}}

\Modelbs\ are a \emph{continuous-time} MAB model with $K$ arms. Other than operating in continuous-time it differs from existing MAB formalisms in the mechanism by which arms generate costs/rewards and the observability of the generated costs from the agent's standpoint. In this section we detail both of these aspects, using the aforementioned cache update scenario as an illustrating example.

\noindent
\textbf{Cost-generating processes.} In \model\ MABs, every arm $k$ incurs a stochastically generated cost $\hat{c}_{k,t}$ \emph{at every time instant} $t$, whether the arm is played or not. However, the distribution of arm $k$'s possible instantaneous costs at time $t$ depends on how much time has passed since the last time arm $k$ was played to refresh the corresponding cached item. We denote the length of this time interval as $\tau_k(t) \in [0, \infty)$. Thus, arm $k$'s cost generation process is described by a family of random variables $\{c_k(\tau_k(t)) | t \geq 0\}$ as stated in the following assumption:

\begin{assumption}
(Cost generation) At time $t$, each arm incurs a cost independently of being played by the agent. The instantaneous cost $\hat{c}_{k,t}$ due to arm $k$ at time $t$ is sampled from a random variable $c_{k,t} = c_k(\tau_k(t))$ s.t.: (1) $\tau_k(0) \triangleq 0$; (2)  $c_k(0) = DiracDelta(0)$; (3) there exists a bound $U < \infty$ s.t. $supp(c_k(\tau)) \subseteq[0, U]$ for every arm's cost generation process $c_k$ and any time interval length $\tau \geq 0$. \\
\label{as:cost_gen}
\end{assumption}
By this assumption, for every arm and any amount of time $\tau$ since its latest play, its cost expectation is well-defined:
$$\overline{c}_k(\tau) \triangleq \mathbb{E}[c_k(\tau)]$$

\noindent
\textbf{Agent's knowledge and cost observability.} While costs are generated by all arms continually, in our model the agent doesn't observe most of them, with an important exception:

\begin{assumption}
(Cost knowledge and observability) For each arm $k$, the agent observes a cost $\hat{c}_{k,t} \sim c_{k,t}$ at time $t$ if and only if the agent plays arm $k$ at that time. The agent doesn't know the distributions of random variables $c_{k,t}$.
\label{as:costobs}
\end{assumption}

Assumption \ref{as:costobs} is crucial in two ways. First, it means that our model provides only bandit feedback. Namely, the agent doesn't see arms' costs at all times, unlike in related models such as maintenance scheduling \cite{bar-noy-soda98}. Second, coupled with Assumption \ref{as:cost_gen} it implies that there is no causal relationship between playing an arm and incurring a cost, which is an implicit assumption that standard bandit strategies rely on.

\noindent
\textbf{Arm play modes.} At any time $t$, any of \model\ MAB's arms can be played in one of two modes:
    
\emph{\textbf{Sync mode.}} Playing arm $k$ in this mode at time $t$ resets the arm's state, i.e., sets $\tau_k(t) \mapsfrom 0$. In addition, per Assumption \ref{as:costobs}, the agent observes the arm's instantaneous cost sample $\hat{c}_{k,t}$ immediately before $\tau_k(t)$ is reset to 0.
    
    In the case of a cache, this means downloading a fresh copy of file $k$, estimating the difference between $k$'s current original and the cached copy, and overwriting the cached copy with the new one.
 
\emph{\textbf{Probe mode.}} By playing arm $k$ allows in probe mode, the agent observes the arm's instantaneous cost, but the arm's state $\tau_k(t)$ is not reset.
    
    In caching settings, this corresponds to downloading a fresh copy of item $k$, but using it purely to estimate the difference between $k$'s current original and the cached copy, without overwriting the cached copy.

Since, by Assumption \ref{as:cost_gen}, $\overline{c}_k(0)= 0$, playing an arm in sync mode gives the agent a way to temporarily reduce the expected rate at which the arm incurs costs. However, due to the following assumption, after a sync play the arm's cost generation rate starts growing again:

\begin{assumption}
(Cost monotonicity) For every arm $k$, the means $\overline{c}_k(\tau_k(t))$ of instantaneous cost random variables $c_k(\tau_k(t))$ are non-decreasing in time since the latest sync-mode play $\tau_k(t)$. If arm $k$ was played in sync mode at time $t_0$, then any sequence of arm $k$'s cost observations $\hat{c}_{k,1}, \hat{c}_{k,2}, \ldots$ yielded by probe-mode plays after $t_0$ and until this arm's next sync-mode play at time $t'_0$ is non-decreasing.
\label{as:cost_incr}
\end{assumption}

Arm state $\tau_k(t)$ can be viewed as the amount of time that has passed since the arm's last sync by time point $t$; the more time has passed, the more cost the arm is incurring per time unit. Playing an arm in sync mode simply resets this time counter. Thus, according to Assumption \ref{as:cost_incr}, not only does the \emph{total} cost generated by arm $k$ since its previous sync play grow as time goes by -- which is to be expected -- but so does the rate at which it happens. 

Note that probe-mode arm plays don't help the agent reduce running costs directly. Instead, as we show in Section \ref{sec:algo}, they help the agent learn a good arm-playing policy faster.

\textbf{Example.} All of the above assumptions are natural in real-world scenarios that inspired the \model\ model. For instance, in Web crawling each online web page accumulates changes according to a temporal process $\mathscr{D}_k(t)$, which is widely assumed to be Poisson (i.e., memoryless) in the Web crawling literature \cite{wolf-www02, cho-vldb02, cho-tds03, cho-tit03, azar-pnas18, kolobov-neurips19, kolobov-sigir19, upadhyay-aaai20}. For each indexed page, the agent (the search engine) incurs a cost $\mathscr{C}_k(d)$ due to serving outdated search results, as a function of the total difference $d$ between the indexed page copy and the online original. From this perspective, $c_k(\tau) \triangleq \mathscr{C}_k(\mathscr{D}_k(\tau))$, but at least one other approach models $c_k(\tau)$ directly as a function of a web page copy's age \cite{cho-icmd00}. In either case, Assumption \ref{as:cost_incr} holds: the more time passes since the page's last crawl, the higher the expected instantaneous penalty. Moreover, penalties don't decrease between two successive crawls of a page: e.g., in case changes are generated by a Poisson process, their number can only grow with time since last refresh, and so can the penalty.
\section{Policies and their cost functions \label{sec:soln}}

In order to derive a learning strategy for \modelbs\ (Section \ref{sec:algo}) and its regret analysis (Section \ref{sec:analysis}), we first derive the necessary building blocks: the cost of an arbitrary policy for this model, the class of policies that will serve as our algorithm's hypothesis space, and parameterized cost functions for policies of this class.

\noindent
\textbf{Policy costs.} Our high-level aim is finding a \model\ policy $\pi$ that has a low expected average cost over an infinite time horizon. Whether a policy $\pi$ is history-dependent, stochastic, or neither, executing it produces a \emph{schedule} $\sigma_k = ((t_{1}, l_{1}), (t_{2}, l_{2}),\ldots)$ for each arm $k$, a possibly infinite sequence of time points $t_{n_k}$ when the arm is to be played and corresponding labels $l_{n_k}$ specifying whether the arm should be played in probe or sync mode at that time. For convenience, WLOG assume that $t_0$ always refers to $t=0$, let $\tau_{n_k} \triangleq t_{n_k} - t_{n_{k}-1}$, and for any finite horizon $H$ let $N_k(H)$ be the index of schedule $\sigma_k$'s largest time point not exceeding $H$:
\begin{equation}
    N_k(H) \triangleq  \begin{cases} argmax_{n \in \mathbb{N}} \{t_{n} \in \sigma_k| t_n \leq H\} \mbox{ if such $n$ exists}\\
    \infty \mbox{ otherwise}
    \end{cases}
		\label{eq:N}
\end{equation}
Given this definition, let $t_{N_k(H)+1} \triangleq H$.

Recalling that each arm has a specific time-dependent cost distribution $c_k(\tau_k(t))$ with mean $\overline{c}_k(\tau_k(t))$, we define the average infinite-horizon cost $J^{\sigma_k}_k$ of arm $k$'s schedule $\sigma_k$ as 
\begin{align}
    J^{\sigma_k}_k &\triangleq \lim \inf_{H \rightarrow \infty} \mathbb{E}\left[\frac{1}{H} \sum_{n_k=1}^{N_k(H) + 1} \int_0^{\tau_{n_k}} c_k(\tau)d\tau\right] \nonumber\\
    &= \lim \inf_{T \rightarrow \infty} \frac{1}{H} \sum_{n_k=1}^{N_k(H) + 1} \int_0^{\tau_{n_k}} \overline{c}_k(\tau)d\tau
\end{align}
Letting
\begin{equation}
    \overline{C}_k(\tau') \triangleq \int_0^{\tau'}\overline{c}_k(\tau)d\tau,
		\label{eq:C}
\end{equation}
we can rewrite $J^{\sigma_k}_k$'s definition as
\begin{equation}
J^{\sigma_k}_k = \lim \inf_{H \rightarrow \infty} \frac{1}{H} \sum_{n_k=1}^{N_k(H) + 1} \overline{C}_k(\tau_{n_k})
\label{eq:Jk}
\end{equation}
Here, $\overline{C}_k(\tau_{n_k})$ is the total cost that arm $k$ is expected to incur between $(n_k-1)$-th and $n_k$-th plays according to schedule $\sigma_k$. Thus, $J^{\sigma_k}_k$ is just the average of these costs over the entire schedule. If the schedule stops playing arm $k$ forever after some time $t$, $J^{\sigma_k}_k$ may be infinite.

Running a policy $\pi$ amounts to sampling a joint schedule $\bm{\sigma} = \{\sigma_k\}_{k=1}^K$. Therefore, we define \emph{policy cost} $J^\pi$ as

\begin{align}
    J^{\pi} \triangleq& \mathop{\mathbb{E}}\limits_{\bm{\sigma} \sim \pi} \left[\frac{1}{K} \sum_{k=1}^K J^{\sigma_k}_k\right] \label{eq:pol_cost}\\
    =& \mathop{\mathbb{E}}\limits_{\bm{\sigma} \sim \pi}\left[\frac{1}{K}\sum_{k=1}^K \left[\lim \inf_{H \rightarrow \infty} \frac{1}{H} \sum_{n_k=1}^{N_k(H) + 1} \overline{C}_k(\tau_{n_k})\right]\right] \nonumber
\end{align}
\\
\noindent
\textbf{Target policy class.} Instead of considering all possible \model\ policies as potential solutions, in this paper we focus on those whose sync-mode plays are \emph{periodic}, with equal gaps between every two consecutive such plays of a given arm. For arm $k$, we denote the length of these gaps as $1/r_k > 0$ length, $r_k$ being a policy parameter for this arm and $\bm{r} \triangleq (r_k)_{k=1}^K$ being the joint parameter vector for all arms. Importantly, our policies do allow probe-mode arm plays but don't restrict how the time points for these plays are chosen. In particular, they may be chosen stochastically, as long the timings of sync-mode plays are deterministically periodic.

Formally, for a scheduled arm pull time point $t_{n_k}$ in schedule $\sigma_k$, let $NextSync_{\sigma_k}(t_{n_k})$ be the next sync-mode play of arm $k$ in $\sigma_k$, i.e., $t_{n'_k} = NextSync_{\sigma_k}(t_{n_k})$ if $t_{n_k} = \min\{t_{n''_k} \mid n''_k > n_k, (t_{n''_k}, l_{n''_k}) \in \sigma_k, l_{n''_k} = sync\}$. Then our target policy class is 
\begin{align}
\Pi = \{\pi(\bm{r})\ |\ & \forall [\bm{\sigma} \sim \pi(\bm{r}), k\in[K] \mbox{, and } (t_{n_k}, l_{n_k}) \in \sigma_k \nonumber\\
&\mbox{s.t. } l_{n_k} = sync \mbox{ and } t_{n'_k} = NextSync_{\sigma_k}(t_{n_k})]\nonumber\\
&t_{n'_k} - t_{n_k} = \frac{1}{r_k} \mbox{ for  $r_k \in \bm{r}$}\}
\label{eq:pol}
\end{align}

Parameters $\bm{r}$ can be interpreted as rates at which arms are played in sync mode. For $\pi \in \Pi$, policy costs are uniquely determined by sync-mode play rates $\bm{r}$: although these parameters ignore the timing of probe plays, probe plays don't affect cost generation and therefore policy cost. 

We let $J(\bm{r})$ denote $\pi(\bm{r}) \in \Pi$'s policy cost. Equation \ref{eq:pol_cost} implies that its cost functions $J(\bm{r})$ have a special form critical for our regret analysis in Section \ref{sec:analysis} -- they are \emph{convex}:

\begin{lemma}
For any policy $\pi(\bm{r}) \in \Pi$, 
\begin{equation}
    J(\bm{r}) = \frac{1}{K} \sum_{k=1}^K r_k \overline{C}_k\left(\frac{1}{r_k}\right).
    \label{eq:pol_cost_r}
\end{equation}
$J(\bm{r})$ and $J_k(r_k) \triangleq r_k\overline{C}_k\left(\frac{1}{r_k}\right)$ for each $k \in [K]$ is convex and monotonically decreasing for $\bm{r} > \bm{0}$.
\label{lem:convex}
\end{lemma}

\begin{proof}
See the Appendix.
\end{proof}

The convexity of the cost functions plays a crucial role in obtaining the regret bounds (Section \ref{sec:analysis}) for the policy learning algorithm presented in Section \ref{sec:algo}.

\noindent
\textbf{Policy constraints.} Naturally, we would like to find a $\pi(\bm{r}) \in \Pi$ that minimizes $J(\bm{r})$ (Equation \ref{eq:pol_cost_r}). As described, however, $\Pi$ has no such policy: note that $\lim_{\bm{r} \rightarrow \bm{\infty}} J(\bm{r}) = \bm{0}$, but no finite $\bm{r}$ attains this limit. However, in practical applications the rates $r_k$ cannot be arbitrarily high individually or in aggregate, and are subject to several constraints. Therefore, in this paper we regard feasible $r_k$ as bounded from above for all $k$ by a universal bound $r_{max}$. Moreover, we assume that the sum of all arms' play rates may not exceed some value $B > 0$. E.g., in Web crawling $B$ is commonly interpreted as a limit on crawl rate imposed by physical network infrastructure \cite{azar-pnas18,kolobov-neurips19, upadhyay-aaai20}. Last but not least, valid $r_k$ values may not be 0, since this implies never playing this arm after a certain time point. In applications such as Web crawling, this means abandoning a cached item (e.g., an indexed webpage) to grow arbitrarily stale, which is unacceptable, so we impose a minimum sync-mode play rate $r_{min}$ on every arm. Note that since, by Assumption \ref{as:cost_gen}, every $c_k(\tau)$ is bounded for $\tau \geq 0$, every $J_k(r_k)$ (Lemma \ref{lem:convex}) is bounded as well.

\textbf{Policy optimization.} Thus, \emph{if we knew cost processes $c_k(.)$}, we could use them to compute $\overline{C}_k(.)$ for all arms and would face the following optimization problem:

\begin{problem}[\modelb\ instance]
\begin{align}
    \mbox{\textbf{Minimize:}}& \ J(\bm{r}) = \frac{1}{K} \sum_{k=1}^K r_k \overline{C}_k\left(\frac{1}{r_k}\right) \label{eq:convex}  \\
    \mbox{\textbf{subject to:}}& \ \bm{r}\in\mathcal{K} \triangleq \left\{\bm{r'}\in[r_{min},r_{max}]^K\ \mid\ ||\bm{r'}||_1=B\right\} \nonumber
\end{align}
\label{p:sb}
\end{problem}
Notice that this formulation implicitly assumes that the entire bandwidth $B$ will be used for sync-mode arm plays -- indeed, if the model is known, there is no need for probes. 

As a side note, we remark that the class of periodic policies $\Pi$ doesn't restrict Problem \ref{p:sb}'s solution quality compared to broader the class of deterministic open-loop policies. We state it here informally as a proposition, which we reformulate more precisely and prove in the Appendix \ref{sec:proofs}: 

\begin{proposition} \label{prop:opt}
For a given $K$-armed \modelb\ instance (Problem \ref{p:sb}), the optimal \emph{periodic} policy $\pi^* \in \Pi$ has the same cost $J^*$ as the optimal \emph{general deterministic open-loop} policy under the same constraints.
\end{proposition}
\section{Online learning for cache synchronization \label{sec:algo}}

In reality we don't know the cost generation processes and can't solve the above optimization problem directly. Instead, we adopt an online perspective on \model\ bandits. A key contribution of this paper that we present in this section is \algo\ (Algorithm \ref{alg:mirror}), an algorithm that treats a \model\ MAB as an online learning problem. A \algo\ agent can be viewed operates in rounds $\mathcal{T} = 1, 2, \dots\ \mathcal{T}_{max}$,  in each round ``playing" a candidate policy parameter vector $\bm{r}^{\mathcal{T}}$, suffering a ``loss'' $\hat{J}^{\mathcal{T}} \sim  J^{\mathcal{T}}(\bm{r}^\mathcal{T})$, and updating $\bm{r}^{\mathcal{T}}$ to a new vector $\bm{r}^{\mathcal{T}+1}$ as a result. As we show in Section \ref{sec:analysis}, \algo\ has an adversarial regret of $O(\mathcal{T}_{max}^{2/3})$.

\algo's novelty is due to the fact that, despite superficial similarities to standard online learning, our setting is different from it in crucial ways, and \algo\ circumvents these differences:

\textbf{(1)} Although the agent can be viewed as suffering loss $\hat{J}_{\mathcal{T}}$, it doesn't \emph{observe} this loss. By \model\ MAB's assumptions, it observes only samples of instantaneous costs $c_k(.)$, and only when it plays arms. Existing techniques don't offer a way to estimate the gradient $\nabla J$ in this case. Moreover, even these impoverished observations take real-world time to collect. \textbf{(2)} In online learning, regret analysis normally assumes $\nabla J$ to be bounded. This isn't quite the case in our model. While we could assume a bound on $\nabla J$ linear in $1/r_{min}$, it would be detrimental to the regret bound when $1/r_{min}$ is large. We show how \algo\ addresses challenge \textbf{(1)} in this section, and circumvent challenge \textbf{(2)} in Section \ref{sec:analysis}.      

\textbf{A note on infinite vs. finite-horizon policies.} The policy cost functions we derived in Section \ref{sec:soln} describe the steady-state performance of a policy over an \emph{infinite} time horizon. However, \algo's rounds are finite. Thus, the cost function $J$ (Equation \ref{eq:pol_cost_r}) that \algo\ uses as a basis for policy improvement is a proxy measure for the average costs that running \algo\ incurs in each round. 

\textbf{\algo\ operation.} At a high level, \algo's main insight is allocating a small fraction of available bandwidth $B$, determined by input parameter $\varepsilon$ (Algorithm \ref{alg:mirror}), to probe-mode arm plays, while using the bulk $\left(\frac{1}{1+\varepsilon}\right)B$ of the bandwidth (line \ref{l:spb}) to play in sync mode according to the current rates $\bm{r}$. \algo\ uses instantaneous cost samples obtained from both to estimate the gradient $\nabla J$ (lines \ref{l:ge_start} - \ref{l:ge_end}) by individually estimating its partial derivatives (lines \ref{l:g_start}-\ref{l:g_end}), which we denote as $$\partial_k J \triangleq \frac{\partial J}{\partial r_k}$$ for short. At the end of each epoch, it does online mirror descent on these estimates to get a new sync-mode policy $\bm{r}$ (lines \ref{l:m_invoke}, \ref{l:m_start}-\ref{l:m_end}).

\begin{algorithm2e}
\small
\DontPrintSemicolon
\SetCommentSty{textrm}
\KwIn{$r_{min}$ -- lowest allowable arm play rate \\
\qquad\ \ \ \ $r_{max}$ -- highest allowable arm play rate \\
\qquad\ \ \ \ $B$ -- bandwidth \\ 
\qquad\ \ \ \ $\varepsilon$ -- probability of probe-mode arm play\label{l:epsilon}\\
\qquad\ \ \ \ $\eta$ -- learning rate \\
\qquad\ \ \ \ $\mathcal{T}_{max}$ -- number of rounds}
\KwOut{$\bm{r}$ -- arm play rates.}
\vspace{7pt}
$T_{round} \mapsfrom 1 / r_{min}$ \ \ \ {\color{gray}// Round length \label{l:rl}\;}
$\mathcal{K}_{\varepsilon} \mapsfrom \left\{\mathbf{x}\in[r_{min},\frac{r_{max}}{1+\varepsilon}]^K \bigg| ||\mathbf{r}||_1=\frac{B}{1+\varepsilon}\right\}$\label{l:spb}\;

$\bm{r} \mapsfrom arg \min_{\bm{x} \in \mathcal{K}_\varepsilon}\mbox{\textbf{BarrierF}}(\bm{x})$  \ \ \ {\color{gray}// Initialize play rates\;}

\ForEach{round $\mathcal{T} = 1, \ldots, \mathcal{T}_{max}$}
{
{\color{gray}// At the start of each round, all arms are assumed \;}
{\color{gray}// to be synchronized and time re-starts at 0.\;}
\ForEach{arm $k \in [K]$}
{
{\color{gray}// Construct a schedule $\sigma^{\mathcal{T}}_k$ for the $\mathcal{T}$-th round.\;} 
$t_{prev, k}\mapsfrom 0$\label{l:tzero}\;
$\sigma_k^{\mathcal{T}}, t_{prev, k} \mapsfrom \mbox{\textbf{ScheduleArmPlays}}(t_{prev, k}, r_k,  T_{round})$ 
}
\ForEach{arm $k \in [K]$ simultaneously \label{l:ge_start}}
{
{\color{gray}// Sample costs by playing according to $\sigma_k^{\mathcal{T}}$\;}
$(\hat{c}_{k, t_1}, \ldots, \hat{c}_{k,t_{|\sigma_k^{\mathcal{T}}|}}) \mapsfrom Play(\sigma_k)$ \label{l:play}\;
\ForEach{$n = 1, \ldots, |\sigma_k^{\mathcal{T}}|\ \mbox{and}\ (t_n, l_n) \in \sigma_k^{\mathcal{T}}$\label{l:gek_start}}
{
\lIf{$l_{n}\ \shorteq \shorteq\ sync$}{$\hat{c}^{(\epsilon)} \mapsfrom none, \hat{c} \mapsfrom \hat{c}_{k,t_n}$}
\lElse{$\hat{c}^{(\epsilon)} \mapsfrom \hat{c}_{k,t_n}, \hat{c} \mapsfrom \hat{c}_{k,t_{n+1}}, n \mapsfrom n+1$}
$\hat{g}_{k} \mapsfrom \frac{\mbox{\textbf{GradJSample}}(\hat{c}^{(\epsilon)}_k, \hat{c}_k, r_k, \varepsilon)}{K}$\;
\textbf{break}\label{l:gek_0}\;
}
}
$\bm{\hat{g}_{\mathcal{T}}} \mapsfrom (\hat{g}_1, \ldots, \hat{g}_K)$\label{l:ge_end}\;
$\bm{r} \mapsfrom \mbox{\textbf{MirrorDescentStep}}(\eta, \mathcal{K}_{\varepsilon}, \bm{\hat{g}_{\mathcal{T}}},\bm{r})$\label{l:m_invoke}\;
}
Return $\bm{r}$\;
\ \\
$\mbox{\textbf{GradJSample}}(\hat{c}^{(\epsilon)}_k, \hat{c}_k, r_k, \varepsilon)$:\label{l:g_start}\;
\lIf{$\hat{c}^{(\epsilon)}_k\ \shorteq \shorteq\ none$ \label{l:gek_end}}{Return  $0$}
\lElse{Return $\frac{1}{\varepsilon r_k }(\hat{c}^{(\epsilon)}_k - \hat{c}_k) $\label{l:g_end}}
\ \\
$\mbox{\textbf{ScheduleArmPlays}}(t_{prev, k}, r_k,  interval\_len)$:\label{l:sch_s}\;
$\sigma_k \mapsfrom (),\ {n_k} \mapsfrom 0$, $t_0\mapsfrom t_{prev, k}$\;
\While{$t_{n_k} + 1/r_k < interval\_len$\label{l:cs}}
{
$t_{prev\_sync} \mapsfrom t_{n_k}$\;
$n_k \mapsfrom n_k + 1$\;
$Probe \sim \mbox{Bernoulli}(\varepsilon)$\label{l:probe_e}\;
\If{$Probe$ \label{l:probe_ins_s}}
{
$t_{n_k} \sim \mbox{Uniform}(t_{n_k - 1}, t_{n_k - 1} + 1/r_k)$\label{l:probe_u}\;
$\sigma_k \mapsfrom Append(\sigma_k, (t_{n_k}, probe))$\;
$n_k \mapsfrom n_k + 1$\;
}\label{l:probe_ins_e}
$t_{n_k} \mapsfrom t_{prev\_sync} + 1/r_k$\label{l:sp_s}\;
$\sigma_k \mapsfrom Append(\sigma_k, (t_{n_k}, sync))$\label{l:sp_e}\;
}
$t_{prev, k} \mapsfrom t_{n_k}$\;
Return $\sigma_k, t_{prev, k}$\;
\ \\
$\mbox{\textbf{MirrorDescentStep}}(\eta, \mathcal{K}, \bm{\overline{g}}, \bm{r}):$\label{l:m_start}\;
Return $\argmin_{\bm{x} \in \mathcal{K}}\{ \eta \langle \bm{x}, \bm{\overline{g}}\rangle +\textbf{DivF}(\bm{x},\bm{r})\}$\label{l:m_end}\;
\ \\
$\mbox{\textbf{DivF}}(\bm{x},\bm{r}):$
Return $\sum_{k=1}^K -\log(x_k/r_k)+x_k/r_k-1$\;
\ \\
$\mbox{\textbf{BarrierF}}(\bm{r}):$
Return $\sum_{k=1}^K -\log(r_k)$\;
\caption{\algo}
\label{alg:mirror}
\end{algorithm2e}

In more detail, in the spirit of online learning, \algo\ assumes that at the start of each round all arms' cost generation processes have just been reset to $c_k(0)$, and restarts the time counter at $t=0$ for every arm (line \ref{l:tzero}). (In practice, this assumption is unrealistic, and we lift it in another variant of \algo\ in Section \ref{sec:disc}.) Then, for every arm $k$, it schedules sync-mode plays at intervals $1/r_k$ (lines \ref{l:cs}, \ref{l:sp_s}-\ref{l:sp_e}) until the end of the current round, and attempts to insert one probe-mode play into each such interval (lines \ref{l:probe_ins_s}-\ref{l:probe_ins_e}) independently with probability $\varepsilon$ (line \ref{l:probe_e}), at a point chosen uniformly at random over the interval's length (line \ref{l:probe_u}). Executing the constructed schedule (line \ref{l:play}) yields cost samples that are used in the aforementioned gradient estimation, which, crucially, is unbiased:

\begin{lemma}
For a rate vector $\bm{r} = (r_k)_{k=1}^K$ and a probability $\varepsilon$, suppose the agent plays each arm in sync mode $1/r_k$ time after that arm's previous sync-mode play, observing a sample of instantaneous cost  $\hat{c}_k \sim c_k(1/r_k)$. Suppose also that in addition, with probability $\varepsilon$ independently for each arm $k$, the agent plays arm $k$ in probe mode at time $t \sim \mbox{Uniform}[0, 1/r_k]$ after that arm's previous sync-mode play, observing a sample of instantaneous cost  $\hat{c}^{(\epsilon)}_k \sim c_k(t)$. Then for each $k$,
\begin{equation*}
g_k \triangleq \begin{cases}
0 &\text{if $\lnot\mbox{Bernoulli}(\varepsilon)$}\\
\frac{1}{\varepsilon r_k K}(\hat{c}^{(\epsilon)}_k -\hat{c}_k)  &\text{if  $\mbox{Bernoulli}(\varepsilon)$}
\end{cases}
\end{equation*}
is an unbiased estimator of $\partial_k J(r_k)$.
\label{lem:gradient estimators}
\end{lemma}
\begin{proof} See the Appendix.
\end{proof}

In one round, \algo\ may get several gradient estimates for a given arm, in which case it takes the first one (line \ref{l:gek_0}). To get a regret bound, however, it is crucial to ensure that for each arm \algo\ receives \emph{at least one} such estimate per round, even if the estimate is 0 (line \ref{l:gek_0}). This is why we set the length of each round to be $1/r_{min}$ (line \ref{l:rl}) --- the largest value $1/r_{k}$ and hence the longest time that \algo\ may have to wait in order to get a cost sample at $1/r_k$.

\section{Regret analysis \label{sec:analysis}}
We generalize our stochastic setting to an adversarial problem and prove an adversarial regret bound of order $\mathcal{O}\left(T^\frac{2}{3}\right)$. This means that the cost distributions $c_k$ and all derived quantities ($\bar c_k$, $\bar C_k$, $J_k$) need not be non-stationary from one round to the next. The cost distributions and derived functions at round $\mathcal{T}$ are denoted by $c_k^\mathcal{T}$, $\bar c_k^\mathcal{T}$, $\bar C_k^\mathcal{T}$, $J_k^\mathcal{T}$ and can be chosen by an oblivious adversary ahead of time.

\textbf{Regret.}
We define the regret with respect to a fixed schedule $\bm{r}\in [0,\infty)^d$ by
\begin{align*}
    \Reg(\bm{r}) \triangleq \mathbb{E}\left[\sum_{\mathcal{T}=1}^{\mathcal{T}_{max}}J^\mathcal{T}(\bm{r}^\mathcal{T})\right]-\sum_{\mathcal{T}=1}^{\mathcal{T}_{max}}J^\mathcal{T}(\bm{r})\,,
\end{align*}
where the expectation is over the randomness of observed costs $\hat c_k$ and $\bm{r}^\mathcal{T}$ is the choice of the algorithm in round $\mathcal{T}$.
Our goal is to compete with the best possible schedule $\bm{r}^*$ using the full available budget:
\begin{align*}
    \Reg = \Reg(\bm{r}^*)\,,\mbox{ where }
    \bm{r}^* \triangleq \min_{\bm{r}\in \mathcal{K}_0} \sum_{\mathcal{T}=1}^{\mathcal{T}_{max}}J^{\mathcal{T}}(\bm{r}),
\end{align*}
where $\mathcal{K}_0$ is $\mathcal{K}_\epsilon$  (line \ref{l:spb} of Algorithm \ref{alg:mirror}) with $\epsilon=0$.

Since we are not be able to obtain any information on the function value or gradient of $J^{\mathcal{T}}$ without an allocated exploration, we also define the best possible schedule $\bm{r}^*_\varepsilon$ given an exploration constrained by $\varepsilon$ (lines \ref{l:probe_e} - \ref{l:probe_u} of Algorithm \ref{alg:mirror}):
\begin{align*}
    \bm{r}^*_\varepsilon \triangleq \argmin_{\bm{r}\in \mathcal{K}_\varepsilon} \sum_{\mathcal{T}=1}^{\mathcal{T}_{max}}J^{\mathcal{T}}(\bm{r})\,.
\end{align*}
The regret can be decomposed into
\begin{align*}
    \Reg = \underbrace{\Reg(\bm{r}^*_\varepsilon)}_{\mbox{in-policy regret}} + \underbrace{\sum_{\mathcal{T}=1}^{\mathcal{T}_{max}}(J^\mathcal{T}(\bm{r}^*_\varepsilon)-J^\mathcal{T}(\bm{r}^*))}_{\mbox{exploration gap}}\,,
\end{align*}
which we bound separately.

\textbf{In-policy regret.}
The problem is an instance of online learning, but online learning literature typically assumes that the gradients of the objective functions $\nabla J^{\mathcal{T}}$ are uniformly bounded w.r.t. some norm. 
Our setting differs in a key aspect: the gradients $\partial_k J(\bm{r})$ scale proportionally to $r_k^{-1}$.

A naive solution would be to bound $\partial_k J(\bm{r}) \lesssim r_{min}^{-1}$ and use gradient descent.
However, the regret would scale with $r_{min}^{-1}$, which might be prohibitively large.

We show that mirror descent with a carefully chosen potential, namely the \emph{log} barrier $F(\bm{r}) = \sum_{k=1}^K\log(r_k)$, adapts to the geometry of the gradients and replaces the polynomial dependency on $r_{min}^{-1}$ by a $\log$ dependency.

\begin{theorem}
\label{thm:in-policy regret}
For any sequence of convex functions $(J^{\mathcal{T}})_{\mathcal{T}=1}^{\mathcal{T}_{max}}$ and learning rate $0<\eta<\frac{K\varepsilon}{2U}$, the in-policy regret of \algo\ is bounded by
\begin{align*}
    \Reg(\bm{r}^*_\varepsilon) \leq \frac{K}{\eta}\log\left(\frac{B}{r_{min}K}\right)+ \eta\frac{U^2}{\varepsilon K}\mathcal{T}_{max}\,.
\end{align*}
\end{theorem}
\begin{proof}
See the Appendix.
\end{proof}
\textbf{Exploration Gap.}
We show that the exploration gap scales proportionally to $\varepsilon$ and is independent of $r_{min}^{-1}$.
On first sight, this bound is surprising because the exploration gap should be approximately $\langle  \sum_{\mathcal{T}=1}^{\mathcal{T}_{max}}\nabla J^\mathcal{T}(\bm{r}^*),\bm{r}^*-\bm{r}^*_{\varepsilon}\rangle$ and the gradients $\nabla J_k^\mathcal{T}(\bm{r}^*)$ could be unbounded (i.e. of order $\bm{r}^{-1}_{min}$).
The high-level idea behind the following lemma is the observation that at the optimal point $\bm{r}^*$, the gradients in all coordinates must coincide and hence the gradient cannot be of order $r^{-1}_{min}$ even if $r^*_k = r_{min}$.
\begin{lemma}
\label{lem:exploration gap}
The exploration gap is bounded by
\begin{align*}
    \sum_{\mathcal{T}=1}^{\mathcal{T}_{max}}(J^{\mathcal{T}}(\bm{r}^*_\varepsilon)-J^{\mathcal{T}}(\bm{r}^*))\leq 2\varepsilon U\mathcal{T}_{max}\,.
\end{align*}
\end{lemma}
\begin{proof}
See the Appendix.
\end{proof}

Finally we are ready to present the main result.
\begin{corollary}
The regret of \algo\ with $\eta =  \frac{K}{U}\sqrt{\log\left(\frac{B}{r_{min}K}\right)\frac{\varepsilon}{\mathcal{T}_{max}}}$ and 
$\varepsilon = \mathcal{T}_{max}^{-\frac{1}{3}}\log^\frac{1}{3}\left(\frac{B}{r_{min}K}\right)$
is bounded for any $\mathcal{T}_{max}>8\log\left(\frac{B}{r_{min}K}\right)$ by
\begin{align}
    \Reg \leq 3U\mathcal{T}_{max}^\frac{2}{3}\log^\frac{1}{3}\left(\tfrac{B}{r_{min}K}\right)\,. \label{eq:reg}
\end{align}
\label{cor:t23}
\end{corollary}
\vspace{-0.3in}
\begin{proof}
The choice of $\eta,\epsilon$ and the bound on $\mathcal{T}_{max}$ ensure that we can apply Theorem~\ref{thm:in-policy regret} to bound the in-policy regret.
The in-policy regret simplifies to
\begin{align*}
    \Reg(\bm{r}^*_\varepsilon) &\leq 2U\sqrt{\frac{\mathcal{T}_{max}}{\varepsilon}\log\left(\frac{B}{r_{min}K}\right)}\,.
\end{align*}
Adding the exploration gap from Lemma~\ref{lem:exploration gap} and substituting the value for $\varepsilon$ completes the proof.
\end{proof}
\section{Making \algo\ practical \label{sec:disc}}

Although \algo\ lends itself to theoretical analysis, several design choices make its vanilla version impractical. 
\textbf{(1)} \algo\ assumes that all arms are synchronized ``for free'' at the start of each round so that each round starts in the same ``state", which is unrealistic.
\textbf{(2)} \algo\ waits until the end of each $1/r_{min}$-long round to update all arms' play rates simultaneously, which could be months in applications like Web crawling. \textbf{(3)}  \algo's further source of inefficiency is that even if arm $k$ has produced several $\partial_k J(r_k)$ samples in a given round, \algo\ uses only one of them. \algoprac\ (Algorithm \ref{alg:mirrorprac}), which can be viewed as a practical implementation of \algo, mitigates these weaknesses of \algo.

\begin{algorithm2e}
\small
\DontPrintSemicolon
\SetCommentSty{textrm}
\KwIn{$r_{min}$ -- lowest allowable arm play rate \\
\qquad\ \ \ \ $r_{max}$ -- highest allowable arm play rate \\
\qquad\ \ \ \ $B$ -- bandwidth \\ 
\qquad\ \ \ \ $\varepsilon$ -- probability of probe-mode arm play\\
\qquad\ \ \ \ $\eta$ -- learning rate \\
\qquad\ \ \ \ $T_{max}$ -- (wall-clock) time horizon \label{l:hor}\\
\qquad\ \ \ \ $\mathcal{S}=(t^{(upd)}_1, t^{(upd)}_2,...)$ -- update schedule}
\KwOut{$\bm{r}$ -- arm play rates.}
\vspace{7pt}
$\mathcal{K}_{\varepsilon} \mapsfrom \left\{\mathbf{x}\in[r_{min},\frac{r_{max}}{1+\varepsilon}]^K \bigg| ||\mathbf{r}||_1=\frac{B}{1+\varepsilon}\right\}$\;

$\bm{r} \mapsfrom arg \min_{\bm{x} \in \mathcal{K}_\varepsilon}\mbox{\textbf{BarrierF}}(\bm{x})$  \ \ \  {\color{gray}// Initialize play rates\;}
$t_{now} \mapsfrom 0$\ \ \ // current time\;
{\color{gray}// Extend arms' schedules $\sigma_k$ until the next update time}
\ForEach{arm $k \in [K]$}
{
$t_{prev, k} \mapsfrom t_{now}$\;
$\sigma_k, t_{prev, k} \mapsfrom \mbox{\textbf{ScheduleArmPlays}}(t_{prev, k}, r_k,   t^{(upd)}_1)$ \;
}
{\color{gray}// $t_{now}$ is incremented continuously \;}
\While{$t_{now} \leq T_{max}$}
{
{\color{gray}// Play each arm $k$'s current $\sigma_k$, record cost samples\;}
\ForEach{arm $k \in [K]$ simultaneously}
{
$(\hat{c}_{k, t_1}, \ldots, \hat{c}_{k,t_{|\sigma_k|}}) \mapsfrom Play(\sigma_k)$
}
{\color{gray}// If now is the next update time $i$ ...\;}
\If{$t_{now} \shorteq \shorteq t^{(upd)}_i\mbox{ for some }t^{(upd)}_i \in \mathcal{S}$}
{
$\mathcal{A}_i \mapsfrom \emptyset$ {\color{gray}// Set of arms that will be updated now\;}
{\color{gray}// Collect cost samples since prev. update time\;}
\ForEach{arm $k \in [K]$}
{
\ForEach{$n, (t_n, l_n) \in \sigma_k\ \mid \ t_n \geq t^{(upd)}_{i-1}$ }
{
\If{$l_{n}\ \shorteq \shorteq\ sync$}{$\hat{c}^{(\epsilon)} \mapsfrom none, \hat{c} \mapsfrom \hat{c}_{k,t_n}$}
\Else(\tcp*[h]{{\color{gray}This was a probe play}})
{
$\hat{c}^{(\epsilon)} \mapsfrom \hat{c}_{k,t_n}, \hat{c} \mapsfrom \hat{c}_{k,t_{n+1}}$\; 
$n \mapsfrom n+1$
}
{\color{gray}// Estimate $\partial_k J$ per Lemma \ref{lem:gradient estimators} (Alg. \ref{alg:mirror},\;
// lines \ref{l:g_start}-\ref{l:g_end}) using collected samples\;}
$\hat{g}_{n, k} \mapsfrom \mbox{\textbf{GradJSample}}(\hat{c}^{(\epsilon)}_k, \hat{c}_k, r_k, \varepsilon)$\label{l:avg1}\;
}
\If{we got at least one $\hat{g}_{n, k}$\label{l:as_start} for arm $k$}
{
$\mathcal{A}_i \mapsfrom \mathcal{A}_i \cup\{k\}$\;
$\overline{g}_k \mapsfrom Avg(\{\hat{g}_{n,k} \mid t_n \geq t^{(upd)}_{i-1}\})$\label{l:avg2}\;
}
}
{\color{gray}// Normalize new grad. estimates\;}
$\bm{\overline{g}_{i}} \mapsfrom \frac{1}{|\mathcal{A}_i|}(\overline{g}_k)_{k \in \mathcal{A}_i}$ \;
{\color{gray}// Now, update play rates \emph{only} for arms in $\mathcal{A}_i$\;}
$\mathcal{K}_{\varepsilon,i} \mapsfrom \left\{\mathbf{x}\in[r_{min},\frac{r_{max}}{1+\varepsilon}]^{|\mathcal{A}_i|} \bigg| ||\mathbf{r}||_1=\frac{\sum_{k \in \mathcal{A}_i} r_k}{1+\varepsilon}\right\}$\label{l:loc_con}\;
$\bm{r_{\mathcal{A}_i}} \mapsfrom \mbox{\textbf{MirrorDescentStep}}(\eta, \mathcal{K}_{\varepsilon,i}, \bm{\overline{g}_{i}},\bm{r_{\mathcal{A}_i}})$\label{l:as_end}\;
\ForEach{arm $k \in [K]$}
{
{\color{gray}// Extend sched. $\sigma_k$ until next update time.\;} 
\lIf{$\mbox{Bernoulli}(\varepsilon)$}{$t_{prev, k} \leftarrow t_{now}$\label{l:stitch1}}
\lElse{$t_{prev, k} \leftarrow \max\{t_{prev, k} + \frac{1}{r_k}, t_{now} \}$\label{l:stitch2}}
$\sigma_k \leftarrow Append(\sigma_k, (t_{prev, k}, sync))$\;
$\sigma'_k, t_{prev, k} \mapsfrom \mbox{\textbf{ScheduleArmPlays}}(t_{prev, k}, r_k,  t^{(upd)}_{i+1})$ 
$\sigma'_k \mapsfrom Append(\sigma_k, \sigma'_k)$
}
}
}
Return $\bm{r}$\;
\caption{\algoprac}
\label{alg:mirrorprac}
\end{algorithm2e}

In contrast to \algo, which performs updates in rigidly defined rounds, \algoprac\ updates policy according to a user-specified schedule $\mathcal{S}$ (see Algorithm \ref{alg:mirrorprac}'s inputs). The length of inter-update periods is unimportant for \algoprac, unlike for \algo\ (line \ref{l:rl}, Alg. \ref{alg:mirror}), due to a major difference between the two algorithms. \algo\ aims to update all arms' parameters \emph{synchronously} at the end of each round and makes the rounds very long to \emph{guarantee} that each arm has generated at least one gradient estimate by the end of each round. In the meantime, \algoprac\ does updates \emph{asynchronously}, performing mirror descent at an update time $t^{(upd)}_i \in \mathcal{S}$ only on those arms that happen to have generated at least one new gradient sample since the previous update time $t^{(upd)}_{i-1}$ (lines \ref{l:as_start}-\ref{l:as_end}). \algoprac\ does these local updates while respecting the global constraint $B$ by using the sum of current play rates or arms that are about to be updated as a local constraint (line \ref{l:loc_con}). Thus, \algoprac\ doesn't need to make inter-update intervals excessively long and doesn't suffer from issue (1). 

As a side note, the reason \algo's regret bound in Corollary \ref{cor:t23} has no linear dependence on $1/r_{min}$ is because it characterizes regret in terms of the number of  \emph{rounds}, not wall-clock time. Nonetheless, this dependency matters empirically, and obtaining a wall-clock-time regret bound that is free from it is an interesting theoretical problem.

\algoprac's asynchronous update mechanism also removes the need for ``free" simultaneous sync-mode play of all arms after each round (2). Recall that before each sync-mode play of arm $k$ with probability $\epsilon$ we can play arm $k$ another time, and so far we have chosen to do so in probe mode. The \textbf{ScheduleArmPlays} routine (line \ref{l:sch_s}, Alg. \ref{alg:mirror}) that both \algo\ and \algoprac\ rely on attempts this (lines \ref{l:probe_e}-\ref{l:probe_ins_s}, Alg. \ref{alg:mirror}) for every sync-mode arm play \emph{except} the first arm play of each round. \algoprac\ takes advantage these unused chances to schedule \emph{sync-mode} plays, which reset cost generation for some fraction of arms (line \ref{l:stitch1}, Alg. \ref{alg:mirrorprac}). For the remaining arms, it simply waits until their next sync-mode play (line \ref{l:stitch2}, Alg. \ref{alg:mirrorprac}) to start estimating the new gradient.

Last but not least, \algoprac\ improves the efficiency of updates themselves. It employs all gradient samples we get for an arm between updates, and averages them to reduce estimation variance (lines \ref{l:avg1}, \ref{l:avg2}), thereby rectifying \algo's weakness (3). 

\section{Related work \label{sec:rel_work}}

There are several existing models superficially related to but fundamentally different from \model\ MABs.

In \emph{maintenance job scheduling} \cite{bar-noy-soda98}, as in our setting, each machine (arm) has an associated operating cost per time unit that increases since the previous maintenance, and performing a maintenance reduces this cost temporarily. However, the agent knows all arms' cost functions and always observes the machines' running costs.

\emph{\citet{upadhyay-neurips18}} describe a model for maximizing a long-term reward that is a function of two general marked temporal point processes. This model is more general than \model\ MAB in some ways (e.g., not assuming cost process monotonicity) but allows controlling the policy process's rate only via a policy cost regularization term and provides no performance guarantees.

\emph{Recharging bandits} \cite{immorlica-focs18}, like \model\ MABs, have arms with non-stationary payoffs: the expected arm reward is a convex increasing function of time since the arm's last play. In spite of this similarity, recharging bandits and other MABs with time-dependent payoffs \cite{heidari-aistats16, levine-nips17, cella-aistats20} make the common assumptions that a reward is generated only when an arm is played and that the agent observes all generated rewards. As a result, their analysis is different from ours. In general, payoff non-stationarity has been widely studied in two broad bandit classes. Restless bandits \cite{whittle-ap88} allow an arm's reward distributions to change, but only \emph{independently of} when the arm is played. Rested bandits \cite{gittins-79} also allow an arms' reward distribution changes, but only \emph{when} the arm is played. \Model\ MABs belong to neither class, since their arms' instantaneous costs change both independently and a result of arms being played.

\section{Empirical evaluation \label{sec:exps}}

The goal of our experiments was to evaluate (1) the relative performance of \algo\ and \algoprac, given that  \algo\ assumes ``free" arm resets at the beginning of each round and \algoprac\ doesn't, and (2) the relative performance of \algoprac\ and its version with projected stochastic gradient descent (PSGD) instead of mirror descent, which we denote as \algopracpsgd. The choice of mirror descent instead of PSGD was motivated by the intuition that with mirror descent \algo\ would achieve lower regret than with PSGD (see Section \ref{sec:analysis}). In the experiments, we verify this intuition empirically. Before analyzing the results, we describe the details of our experiment setup.

\textbf{Problem instance generation.} Our experiments in Figures \ref{fig:ms_vs_ams} and \ref{fig:sgd_vs_ams} were performed on \model\ MAB instances generated as follows. Recall from Sections \ref{sec:form} and \ref{sec:soln} that a \model\ MAB instance is defined by:
\begin{itemize}[noitemsep,topsep=0pt]
    \item $r_{min}$, the lowest allowed arm play rate
    \item $r_{max}$, the highest allowed arm play rate
    \item $B$, the highest allowed \emph{total} arm play rate
    \item $K$, the number of arms
    \item $\{c_k(\tau)\}_{k=1}^K$, a set of cost-generating processes --- time-dependent distributions of instantaneous costs, one process for each arm $k$.
\end{itemize}

For all problem instances in the experiments, $r_{min}, r_{max}, B$, and $K$ were as in Table \ref{t:fixed_prob_params} in Appendix \ref{sec:emp_details}. The set $\{c_k(\tau)\}_{k=1}^K$ of cost-generating processes was constructed randomly for each instance. In all problem instances, each arm had a distribution over time-dependent cost functions in the form of polynomials $c_k(\tau) = a_k  \tau^{p_k}$, where $p_k \in (0,1)$ was chosen at instance creation time and $a_k$ sampled at run time as described in Appendix \ref{sec:emp_details}. Note that \algo's regret (Equation \ref{eq:reg}) depends on the cost cap $U$. While our polynomial cost functions are unbounded in general, they are bounded within the $[r_{min}, r_{max}]$ constraint region we are using (Table \ref{t:fixed_prob_params} in Appendix \ref{sec:emp_details}). Within this constraint region, these cost functions are equivalent to $\min\{a_k  \tau^{p_k}, U\}$, where $U = 40$.

In Appendix \ref{sec:emp_add_res} we also describe a different, Poisson process-based family of cost-generating processes, and present experimental results obtained on it in Figures \ref{fig:ms_vs_ams_bp} and \ref{fig:sgd_vs_ams_bp} in that section. Despite that process family being very distinct from the polynomial one, the results are qualitatively similar to those in Figures \ref{fig:ms_vs_ams} and \ref{fig:sgd_vs_ams}.

\textbf{Implementation details.}
We implemented \algo, \algoprac, and \algopracpsgd, along with a problem instance generator that constructs \model\ MAB instances as above, in Python.
The implementation, available at \href{https://github.com/microsoft/Optimal-Freshness-Crawl-Scheduling}{\color{blue}{https://github.com/microsoft/Optimal-Freshness-Crawl-Scheduling}}, relies on \verb|scipy.optimize.minimize| for convex constrained optimization in order to update the play rates $\bm{r}$ (lines \ref{l:m_invoke}, \ref{l:m_start} of Alg. \ref{alg:mirror}, line \ref{l:as_end} of Alg. \ref{alg:mirrorprac}). Other convex optimizers are possible as well. The experiments were performed on a Windows 10 laptop with 32GB RAM with 8 Intel 2.3GHz i9-9980H CPU cores.

\begin{figure*}
\begin{minipage}[t]{0.48\textwidth}
 \vspace{4pt}%

 \vspace{-0.20in}
 \captionof{figure}{\small \algo's and \algoprac's convergence. The two algorithms exhibit nearly identical convergence behavior using tuned hyperparameters. However, \algoprac\ works without assuming the ``free" arm resets after arms' parameter updates that \algo\ relies on.}  
 \label{fig:ms_vs_ams}
 \end{minipage}
 \hspace{0.15in}
\begin{minipage}[t]{0.48\textwidth}
 \vspace{4pt}%
 %
 \vspace{-0.2in}
 \captionof{figure}{\small Asynchronous version of \algo\ with mirror descent vs. with projected SGD. The use of mirror descent with the \emph{log} barrier function in \algo\ was key to constructing the regret bounds in Section \ref{sec:analysis}. Empirically, although \algoprac\ and \algopracpsgd\ eventually converge to similar-quality policies, \algoprac\ discovers good policies faster, as \algo's theory predicts.}
 \label{fig:sgd_vs_ams}
 \end{minipage}
 \vspace{-0.15in}
 \end{figure*}

\textbf{Hyperparameter tuning.} Running the algorithms required choosing values for the following parameters:

\begin{itemize}[noitemsep,topsep=0pt]
    \item{\textbf{Learning rate $\eta$.}} As in other learning algorithms, choosing a good value for $\eta$ for each of the three algorithms was critical for their convergence behavior.
    
    \item{\textbf{Length $l_{upd\_round}$ of intervals between \algoprac's and \algopracpsgd's play rate updates.}} Recall that unlike \algo, which updates all play rates simultaneously after every $1/r_{min}$ time units, \algoprac\ and \algopracpsgd\ update the model parameters according to a user-provided schedule. While the schedule doesn't necessarily have to be periodic, in the experiments it was, with $l_{upd\_round}$ being the inter-update interval length. Intuitively, $l_{upd\_round}$ influences the average number of arms updated during each update attempt and the variance of gradient estimates: the larger $l_{upd\_round}$, the more gradient samples \algoprac\ and \algopracpsgd\ can be expected to average for the upcoming update (line \ref{l:avg2} of Alg. \ref{alg:mirrorprac}). In this respect, $l_{upd\_round}$'s role resembles that of minibatch size in minibatch SGD.  
		
		\item{\textbf{Exploration parameter $\varepsilon$.}} Theory provides a horizon-dependent guidance for setting $\varepsilon$ for \algo\ (Corollary \ref{cor:t23}) but not for \algoprac\ and \algopracpsgd. For comparing \emph{relative} convergence properties of \algo, \algoprac, and \algopracpsgd, we fixed $\varepsilon = 0.05$ for all of them. 
\end{itemize}

\algoprac's and \algopracpsgd's performance is determined by a combination of $\eta$ and $l_{upd\_round}$, so we optimized them together using grid search. Please see Appendix \ref{sec:emp_details} for more details.

\textbf{Experiment results.} Figures \ref{fig:ms_vs_ams} and \ref{fig:sgd_vs_ams} compare the performance of \algo\ vs. \algoprac\ and \algoprac\ vs. \algopracpsgd, respectively. The figure captions highlight important patterns we observed. The plots were obtained by running the respective pairs of algorithms on 150 problem instances generated as above, measuring the policy cost $J$ after every update, and averaging the resulting curves across these 150 trials.

In each trial, all algorithms were run for the amount of simulated time \emph{equivalent to} 240 \algo\ rounds (see Figure \ref{fig:ms_vs_ams}'s and \ref{fig:sgd_vs_ams}'s $x$-axis). However, note that the number of updates performed by \algoprac\ and \algopracpsgd\ was larger than 240. Specifically, a \algo's update round is always of length $1/r_{min}$ time units, but for \algoprac\ and \algopracpsgd\ it is $l_{upd\_round}$ units, so for every \algo\ update round, they perform $(1/r_{min}) / l_{upd\_round}$ updates. Although more frequent model updates is itself a strength of the asynchronous algorithms, their main practical advantage is independence of \algo's ``free arm resets'' assumption.

\vspace{-0.4cm}
\section{Conclusion}

This paper presented \model\ MABs, a bandit class where all arms generate costs continually, independently of being played, and the agent observes an arm's stochastic instantaneous cost only when it plays the arm. We proposed an online learning approach for this setting, called \algo, whose novelty is in estimating the policy cost gradient without directly observing the policy cost function and without having a closed-form expression for it. Moreover, we derived an $O(T^{\frac{2}{3}})$ adversarial regret bound for \algo\ without explicitly requiring the gradients to be bounded. We also presented \algoprac, a practical version of \algo\ that lifts the latter's idealizing assumptions. The key insight behind all these contributions is that the use of mirror descent for policy updates in \model\ MABs enables much faster convergence than gradient descent would. Our experiments confirmed this insight empirically.

\textbf{Acknowledgments.} We would like to thank Nicole Immorlica (Microsoft Research) and the anonymous reviewers for their helpful comments and suggestions regarding this work.
\bibliography{bibl}
\bibliographystyle{icml2020}
\appendix
\onecolumn
 
\ \\
\begin{center}
\textbf{\Large APPENDIX}
\end{center}
\ \\

\section{Proofs \label{sec:proofs}}

\emph{\textbf{Proposition \ref{prop:opt}} For a given $K$-armed \modelb\ instance (Problem \ref{p:sb}), the optimal \emph{periodic} policy $\pi^* \in \Pi$ has the same cost $J^*$ as the optimal \emph{general deterministic open-loop} policy under the same constraints.}

Before proving this result, we formalize its statement. Consider the class $\Pi^o$ of all deterministic open-loop policies:

\begin{equation}
\Pi^o \triangleq \{\pi^o = \{\sigma_k\}_{k=1}^K\ \mid\ \sigma_k=((t_{n_k}, l_{n_k}))_{n_k =1}^{\infty}\},  
\end{equation}

where $l_{n_k} = sync$, denoting synchronization play, and $t_{n_k}$ is a time when arm $k$ is supposed to be played.

Recall that for a schedule $\sigma_k$ and a horizon $H$, $N_k(H)$ is the index of  $\sigma_k$'s largest time point within the horizon $H$. We define the cost $J^{\sigma}_k$ of an arm's schedule $\sigma$ as in Equation \ref{eq:Jk} and let the cost of a policy $\pi \in \Pi^o$ be  
\begin{equation}
J^{\pi} \triangleq \sum_{k=1}^K J_k^{\sigma_k} = \lim \inf_{H \rightarrow \infty} \frac{1}{H} \sum_{k=1}^K \sum_{n_k=1}^{N_k(H) + 1} \overline{C}_k(\tau_{n_k}),
\end{equation}

For an arm $k$ and $r_k \in \mathbb{R}^{+}$, consider

\begin{equation}
    J_k(r_k) \triangleq \inf_{\sigma}\{J^{\sigma}_k\ \mid\ \lim \inf_{H \rightarrow \infty} \frac{N_k(H)}{H} \leq r_k\} \label{eq:armpolcostopt},
\end{equation}

i.e., the limit of the minimum average cost for arm $k$ achievable by pulling arm $k$ at no higher than some average rate $r_k$. We formulate the analog of Problem \ref{p:sb} for policy class $\Pi^o$ \emph{in the known-parameter setting} as follows:

\begin{problem}
\begin{align}
    \mbox{\textbf{Minimize:}}& \ J(\bm{r}) = \frac{1}{K} \sum_{k=1}^K J_k(r_k)  \label{eq:genJ}\\
    \mbox{\textbf{subject to:}}& \ \bm{r}\in\mathcal{K} \triangleq \left\{\bm{r'}\in[r_{min},r_{max}]^K\ \mid\ ||\bm{r'}||_1=B\right\}, \nonumber
\end{align}
\label{p:sb2}
\end{problem}

where $J_k$ is as in Equation \ref{eq:armpolcostopt}. We can now re-state Proposition \ref{prop:opt} formally: \\

\emph{\textbf{Lemma.} For a given $K$-armed \modelb\ instance, the \emph{periodic} policy $\pi^* \in \Pi$ that optimally solves Problem \ref{p:sb} also optimally solves Problem \ref{p:sb2}.}

\begin{proof}
The proof proceeds as follows. First, we derive an expression for $J_k(r_k; H)$, the average cost of the optimal \emph{finite-horizon} schedule that pulls arm $k$ at no higher than rate $r_k$. Then, noting that $J_k(r_k) = \lim_{H \rightarrow \infty}$, we derive an expression for $J_k(r_k)$ (Equation \ref{eq:armpolcostopt}). Finally, using this expression, we show that the policy $\pi^* \in \Pi^o$ that optimally solves Problem \ref{p:sb2} is periodic and hence is also is a solution to Problem \ref{p:sb}. But every solution of Problem \ref{p:sb} is also a solution of Problem \ref{p:sb2}, so $\pi^*$ is optimal for Problem \ref{p:sb} as well. Since, by convexity, Problem \ref{p:sb} has a unique optimal solution, the will follow.

Formally, letting $|\sigma_k|_H$ be the number of schedule $\sigma_k$'s arm plays up to time $H$, we define

\begin{equation}
J_k(r_k; H) \triangleq \inf_{\sigma_k} \{J^{\sigma}_k\ \mid\ \frac{|\sigma_k|_H}{H} \leq r_k\} = \min_{N_k \leq r_k H, N_ \in \mathbb{N}} J_k(H, N_k), \label{eq:JrH}
\end{equation}

where 

\begin{equation}
J_k(H, N_k) \triangleq \inf_{\tau_1, \ldots, \tau_{N_k+1}} \frac{1}{H} \sum_{n_k=1}^{N_k + 1} \overline{C}_k(\tau_{n_k})
\label{eq:JrHN}
\end{equation}

Here, as defined at the beginning of Section \ref{sec:soln}, $\tau_{n_k}$ is the time interval between schedule $\sigma_k$'s $(n_{k}-1)$-th and $n_k$-th arm play ($\tau_{n_k} \triangleq t_{n_k} - t_{n_{k}-1}$,  $t_{N_k+1} =H$, and $\overline{C}_k(\tau_{n_k})  \triangleq \int_0^{\tau_{n_k}}\overline{c}_k(\tau)d\tau$ (Equation \ref{eq:C}).

Now, note that for a fixed $N_k$ and $H$, $\frac{1}{H} \sum_{n_k=1}^{N_k + 1} \overline{C}_k(\tau_{n_k})$ is convex increasing in each $\tau_{n_k}$. This follows because each $\overline{C}_k(\tau_{n_k})$ is convex increasing, which, in turn, follows from $\overline{C}_k(\tau_{n_k})$'s definition because $\overline{c}_k(\tau)$ is positive non-decreasing by Assumption \ref{as:cost_incr}. Therefore, applying the method of Lagrange multipliers shows that, for a fixed $N_k$ and $H$, the minimizer of  $\frac{1}{H} \sum_{n_k=1}^{N_k + 1} \overline{C}_k(\tau_{n_k})$ is the schedule $\sigma_k^* = \{(\frac{nH}{N_k + 1}, sync)\}_{n=1}^{N_k}$, and we can rewrite Equation \ref{eq:JrHN} as

\begin{equation}
J_k(H, N_k) = \frac{N_k +1}{H}C_k\left(\frac{H}{N_k + 1}\right).
\label{eq:JrHN_per}
\end{equation}

Thus, if the number of arm plays $N_k$ is given, the best finite-horizon schedule is periodic. Moreover, Equation \ref{eq:JrHN_per} implies that  $J_k(H, N_k)$ is monotonically decreasing in $N_k$ for a fixed $H$, because for any $N_k \geq 1$, increasing the number of arms plays by 1 always helps reduce the best schedule cost:

\begin{align}
    J^*_k(H, N_k+1)& - J^*_k(H, N_k)\nonumber\\
    &=\frac{N_k +2}{H} \overline{C}_k\left(\frac{H}{N_k + 2}\right) - \frac{N_k +1}{H} \overline{C}_k\left(\frac{H}{N_k + 1}\right)  \nonumber\\ 
    &= \frac{1}{H}\left[ \overline{C}_k\left(\frac{H}{N_k + 2}\right) - (N_k +1)\left(\overline{C}_k\left(\frac{H}{N_k + 1}\right) -  \overline{C}_k\left(\frac{H}{N_k + 2}\right)\right)\right]  \label{eq:eq}\\
    &\leq \frac{N_k +1}{H}\left[\overline{C}_k\left(\frac{H}{N_k + 2}\right) -  \overline{C}_k\left(\frac{H}{N_k + 2} - \frac{H}{(N_k + 2)(N_k + 1)}\right)\right] -  \label{eq:ineq}\\
    & \ \ \ \ \ \ \ \ \ \ \ \ \ \ \ \ \ \ \ \ \ \ \ \ \ \ \ \ \ \ \ \ \ \ \ \ \ \ \ \ \ \ \ \ \ \ -\frac{N_k +1}{H}\left[\overline{C}_k\left(\frac{H}{N_k + 1}\right) - \overline{C}_k\left(\frac{H}{N_k + 2}\right)\right] \leq 0 \label{eq:ineq1}
\end{align}

Equality \ref{eq:eq} was obtained via algebraic manipulations. Inequality \ref{eq:ineq} follows because the interval $[0, \frac{H}{N_k + 2}]$ can be broken down into $(N_k +1)$ intervals of size $\frac{H}{(N_k + 2)(N_k + 1)}$, with cost $\left( \overline{C}_k\left(\frac{H}{N_k + 2}\right) -  \overline{C}_k\left(\frac{H}{N_k + 2} - \frac{H}{(N_k + 2)(N_k + 1)}\right)\right)$ associated with each; since, as pointed out above, $\overline{C}_k(.)$ is convex increasing, we have $(N_k +1)\left(\overline{C}_k\left(\frac{H}{N_k + 2}\right) - C_k\left(\frac{H}{N_k + 2} - \frac{H}{(N_k + 2)(N_k + 1)}\right)\right) \geq  \overline{C}_k\left(\frac{H}{N_k + 2}\right)$ (they are equal if $\overline{C}_k$ is linear). Also due to the convex increasing property of $\overline{C}_k(.)$ and due to $\frac{H}{N_k + 1} > \frac{H}{N_k + 2} > \frac{H}{N_k + 2} - \frac{H}{(N_k + 2)(N_k + 1)}$ and $\frac{H}{N_k + 1} - \frac{H}{N_k + 2} = \frac{H}{N_k + 2} - \left(\frac{H}{N_k + 2} - \frac{H}{(N_k + 2)(N_k + 1)}\right)$, we have $\left(\overline{C}_k\left(\frac{H}{N_k + 2}\right) - \overline{C}_k\left(\frac{H}{N_k + 2} - \frac{H}{(N_k + 2)(N_k + 1)}\right)\right) \leq \left(\overline{C}_k\left(\frac{H}{N_k + 1}\right) - \overline{C}_k\left(\frac{H}{N_k + 2}\right)\right)$, establishing the final inequality \ref{eq:ineq1}.\\

Crucially, $J_k(H, N_k)$ being monotonically decreasing in $N_k$ means that the r.h.s. of Equation \ref{eq:JrH} is minimized when $N_k$ is as large as possible given the consraint $N_k \leq r_k H, N_k \in \mathbb{N}$, i.e., when $N_k = \lfloor r_k H \rfloor$. Plugging $N_k = \lfloor r_k H \rfloor$ into Equation \ref{eq:JrHN_per} and substituting Equation \ref{eq:JrHN_per} into Equation \ref{eq:JrH}, we can rewrite Equation \ref{eq:JrH} as

\begin{equation}
J_k(r_k; H)  = \frac{\lfloor r_k H \rfloor +1}{H}\overline{C}_k\left(\frac{H}{\lfloor r_k H \rfloor + 1}\right)
\end{equation}

Also, notice from Equations \ref{eq:armpolcostopt} and \ref{eq:JrH} that 

\begin{equation}
J_k(r_k) = \lim_{H\rightarrow\infty}J_k(r_k; H)
\end{equation}

This implies that 

\begin{equation}
J_k(r_k) = \lim_{H\rightarrow\infty} \frac{\lfloor r_k H \rfloor +1}{H}\overline{C}_k\left(\frac{H}{\lfloor r_k H \rfloor + 1}\right) = r_k \overline{C}_k\left(\frac{1}{r_k}\right)
\end{equation}

Substituting this expression into Equation \ref{eq:genJ} reveals that Problem \ref{p:sb2} is identical to Problem \ref{p:sb}, so the optimal deterministic periodic policy under Problem \ref{p:sb}'s constraints is also optimal for Problem \ref{p:sb2} over all deterministic open-loop policies.
\end{proof}
\ \\
\ \\
\ \\
\emph{\textbf{Lemma \ref{lem:convex}}
For any policy $\pi(\bm{r}) \in \Pi$, 
\begin{equation}
    J(\bm{r}) = \frac{1}{K} \sum_{k=1}^K r_k \overline{C}_k\left(\frac{1}{r_k}\right).
\end{equation}
$J(\bm{r})$ and $J_k(r_k) \triangleq r_k\overline{C}_k\left(\frac{1}{r_k}\right)$ for each $k \in [K]$ is convex and monotonically decreasing for $\bm{r} > \bm{0}$.
}

\begin{proof}
First we show that the form of $J(\bm{r})$ is as stated in the theorem, and then we show its monotonic decrease and convexity.

Recall from Section \ref{sec:soln} that for arm $k$'s schedule $\sigma_k = ((t_{1}, l_{1}), (t_{2}, l_{2}),\ldots)$, $\{t_{n_k}\}$ is a possibly infinite sequence of time points when the arm is to be played, and for a finite horizon $H$, $N_k(H)$ is the index of schedule $\sigma_k$'s largest time point not exceeding $H$. For convenience we let $t_0 \triangleq 0$, $t_{N_k(H) +1} \triangleq H$, and $\tau_{n_k} \triangleq t_{n_k} - t_{n_k - 1}$ for $n_k \geq 1$.

Consider a cost function for a periodic schedule for arm $k$, where arm $k$ is played at a rat $r_k$. Equation \ref{eq:Jk} implies that it has the form
\begin{align*}
J_k(r_k) &= \lim \inf_{H \rightarrow \infty} \frac{1}{H} \sum_{n_k=1}^{N_k(H) + 1} \overline{C}_k(\tau_{n_k})\\
&= \lim \inf_{H \rightarrow \infty} \frac{1}{H}\left[ \lfloor r_k H \rfloor \overline{C}_k\left(\frac{1}{r_k}\right) + \overline{C}_k\left(H -\lfloor r_k H \rfloor\right)\right]
\end{align*}

Since $H -\lfloor r_k H \rfloor \leq  \frac{1}{r_k}$ and $\overline{C}_k \geq 0$, we have $\overline{C}_k\left(H -\lfloor r_k H \rfloor\right) \leq  \overline{C}_k\left(\frac{1}{r_k}\right)$, so 

\begin{equation*}
J_k(r_k) = \lim \inf_{H \rightarrow \infty} \frac{1}{H}\left[ \lfloor r_k H \rfloor \overline{C}_k\left(\frac{1}{r_k}\right) + \overline{C}_k\left(H -\lfloor r_k H \rfloor\right)\right] = r_k \overline{C}_k\left(\frac{1}{r_k}\right)
\end{equation*}

and hence

\begin{equation*}
    J(\bm{r}) = \frac{1}{K} \sum_{k=1}^K J_k(r_k) = \frac{1}{K} \sum_{k=1}^K r_k \overline{C}_k\left(\frac{1}{r_k}\right),
\end{equation*}

proving the first part of the lemma. 

To show $J$'s convexity, we compute its Hessian:

\begin{align*}
\frac{\partial J}{\partial r_k}  &= \overline{C}_k\left(\frac{1}{r_k}\right) - \frac{1}{r_k} \overline{c}_k\left(\frac{1}{r_k}\right) \\
\frac{\partial^2 J}{\partial r_k ^2} &= - \frac{1}{r_k^2} c_k\left(\frac{1}{r_k}\right) + \frac{1}{r_k^2}\overline{c}_k\left(\frac{1}{r_k}\right) + \frac{1}{r_k^3}\overline{c}'_k\left(\frac{1}{r_k}\right) = \frac{1}{r_k^3}\overline{c}'_k\left(\frac{1}{r_k}\right),
\end{align*}

noting that $\frac{\partial^2 J}{\partial r_j r_k} = 0$ for all $j \neq k$. 

Crucially, by Assumption \ref{as:cost_incr}, $\overline{c}_k$ is non-decreasing, so $\overline{c}_k' \geq 0$ and $\frac{\partial^2 J}{\partial r_k ^2} \geq 0$. Thus, $J$'s Hessian is positive semidefinite, implying that $J$ is convex.

To see that $J$ is monotonically decreasing in each $r_k$, consider $\frac{\partial J}{\partial r_k} = \overline{C}_k\left(\frac{1}{r_k}\right) - \frac{1}{r_k} \overline{c}_k\left(\frac{1}{r_k}\right)$ and note that since $\overline{C}_k\left(\frac{1}{r_k}\right) \triangleq \int_0^{\frac{1}{r_k}} \overline{c}_k(\tau) d\tau$, for any $r_k > 0$ we have $\overline{C}_k\left(\frac{1}{r_k}\right) \leq \frac{1}{r_k} \overline{c}_k\left(\frac{1}{r_k}\right)$ and therefore $\frac{\partial J}{\partial r_k} \leq 0$.
\end{proof}
\ \\
\ \\
\ \\
\emph{\textbf{Lemma \ref{lem:gradient estimators}.} For a rate vector $\bm{r} = (r_k)_{k=1}^K$ and a probability $\varepsilon$, suppose the agent plays each arm in sync mode $1/r_k$ time after that arm's previous sync-mode play, observing a sample of instantaneous cost  $\hat{c}_k \sim c_k(1/r_k)$. Suppose also that in addition, with probability $\varepsilon$ independently for each arm $k$, the agent plays arm $k$ in probe mode at time $t \sim \mbox{Uniform}[0, 1/r_k]$ after that arm's previous sync-mode play, observing a sample of instantaneous cost  $\hat{c}^{(\epsilon)}_k \sim c_k(t)$. Then for each $k$,
\begin{equation*}
g_k \triangleq \begin{cases}
0 &\text{if $\lnot\mbox{Bernoulli}(\varepsilon)$}\\
\frac{1}{\varepsilon r_k K}(\hat{c}^{(\epsilon)}_k -\hat{c}_k)  &\text{if  $\mbox{Bernoulli}(\varepsilon)$}
\end{cases}
\end{equation*}
is an unbiased estimator of $\partial_k J(r_k)$.
}

\begin{proof}
We need to ensure that $\mathbb{E}[\partial_k J(r_k)] - \mathbb{E}\left[g_k\right] = 0$.
By definition of $J_k$ (Lemma \ref{lem:convex}),
\begin{align*}
\mathbb{E}[\partial_k J(r_k)] = \overline{C}_k\left(\frac{1}{r_k}\right) - \frac{1}{r_k} \overline{c}_k\left(\frac{1}{r_k}\right) 
\end{align*}
Similarly, by definition of $g_k$,
\begin{align*}
    \mathbb{E}\left[g_k\right] &= -(1 - \varepsilon) \cdot 0 + \varepsilon \cdot \left(\frac{\mathbb{E}[\hat{c}^{(\epsilon)}_k]-\mathbb{E}[\hat{c}_k]}{\varepsilon r_k}\right)\\
    &= \varepsilon \cdot \left(\frac{1}{\varepsilon r_k}\left(\int_0^{\frac{1}{r_k}} \frac{1}{1/r_k}\overline{c}_k(\tau)d\tau\right) - \overline{c}_k\left(\frac{1}{r_k}\right)\right)\\
    &= \left(\overline{C}_k\left(\frac{1}{r_k}\right) - \overline{C}_k(0)\right) - \frac{1}{r_k}\overline{c}_k\left(\frac{1}{r_k}\right) \\
    &= \overline{C}_k\left(\frac{1}{r_k}\right) - \frac{1}{r_k}\overline{c}_k\left(\frac{1}{r_k}\right)
\end{align*}

The last line follows because $\overline{C}_k(0) = 0$, since, by Assumption \ref{as:cost_gen}, $c_k(0) = DiracDelta(0)$. Thus, $\mathbb{E}[\partial_k J(r_k)] = \mathbb{E}\left[g_k\right]$.
\end{proof}
\ \\
\ \\
\ \\
\emph{\textbf{Lemma \ref{lem:exploration gap}.} 
The exploration gap is bounded by
\begin{align*}
    \sum_{\mathcal{T}=1}^{\mathcal{T}_{max}}(J^{\mathcal{T}}(\bm{r}^*_\varepsilon)-J^{\mathcal{T}}(\bm{r}^*))\leq 2\varepsilon U\mathcal{T}_{max}\,.
\end{align*}
}

\begin{proof}
We define an intermediate set between $\mathcal{K}_\varepsilon$ and $\mathcal{K}_0$ by 
\[
\tilde{\mathcal{K}}_\varepsilon \triangleq \left\{\mathbf{x}\in[r_{min},r_{max}]^K\,\mid\,||\mathbf{r}||_1=(\frac{1}{1+\varepsilon})B\right\}\,,
\]
i.e. the set has the same $\ell_1$ constrain as $\mathcal{K}_\varepsilon$ but the range of $\mathcal{K}_0$.
Recall and define
\begin{align*}
\bm{r}^*_\varepsilon \triangleq \argmin_{\bm r\in \mathcal{K}_\varepsilon}\sum_{\mathcal{T}=1}^{\mathcal{T}_{max}}J^{\mathcal{T}}(\bm r)\,,\qquad \bm{r}^* \triangleq \argmin_{\bm r\in\mathcal{K}_0}\sum_{\mathcal{T}=1}^{\mathcal{T}_{max}}J^{\mathcal{T}}(\bm r)\,,\qquad
\tilde{\bm{r}}^* \triangleq \argmin_{\bm r\in\tilde{\mathcal{K}}_0}\sum_{\mathcal{T}=1}^{\mathcal{T}_{max}}J^{\mathcal{T}}(\bm r)\,.
\end{align*}
We first bound $\sum_{\mathcal{T}=1}^{\mathcal{T}_{max}}(J^\mathcal{T}({\bm{r}}^*_\varepsilon)-J^\mathcal{T}(\tilde{\bm{r}}^*_\varepsilon))$. By convexity we have
\begin{align*}
    \sum_{\mathcal{T}=1}^{\mathcal{T}_{max}}(J^\mathcal{T}({\bm{r}}^*_\varepsilon)-J^\mathcal{T}(\tilde{\bm{r}}^*_\varepsilon)) 
    &\leq \sum_{\mathcal{T}=1}^{\mathcal{T}_{max}}\sum_{k=1}^K \partial_kJ^\mathcal{T}({\bm{r}}^*_\varepsilon) (r^*_{\varepsilon k}-\tilde{r}^*_{\varepsilon k})\\
    &\leq \sum_{\mathcal{T}=1}^{\mathcal{T}_{max}}\sum_{k:\tilde{r}^*_{\varepsilon k}>{\bm{r}}^*_{\varepsilon k}} \partial_kJ^\mathcal{T}({\bm{r}}^*_\varepsilon) (r^*_{\varepsilon k}-\tilde{r}^*_{\varepsilon k})\,,
\end{align*}
where the last line follows from the negativity of all gradients.
We note that $\tilde{r}^*_{\varepsilon k}>{\bm{r}}^*_{\varepsilon k}$ implies ${\bm{r}}^*_{\varepsilon k}=\frac{r_{max}}{1+\epsilon}$, because otherwise one of the extreme points violates the K.K.T. conditions listed below.
The gradients are monotonically increasing, so we have
\begin{align*}
    \sum_{\mathcal{T}=1}^{\mathcal{T}_{max}}\sum_{k:\tilde{r}^*_{\varepsilon k}>{\bm{r}}^*_{\varepsilon k}} \partial_kJ^\mathcal{T}({\bm{r}}^*_\varepsilon) (r^*_{\varepsilon k}-\tilde{r}^*_{\varepsilon k})
    &\leq \sum_{\mathcal{T}=1}^{\mathcal{T}_{max}}\sum_{k=1}^K \partial_kJ^\mathcal{T}\left(\frac{r_{max}}{1+\epsilon}\right)\left(\frac{r_{max}}{1+\epsilon}-r_{max}\right)\\
    &\leq K \left(-\frac{U(1+\epsilon)}{r_{max}K}\right)\left(-\frac{\epsilon r_{max}}{1+\epsilon}\right)\mathcal{T}_{max}=\epsilon U \mathcal{T}_{max}\,. 
\end{align*}
Now we bound the gap between $\bm{r}^*$ and $\tilde{\bm{r}}^*_\varepsilon$.
The functions $J^{\mathcal{T}}$ are convex and monotonically decreasing according to Lemma~\ref{lem:convex}. By convexity and Cauchy-Schwarz, it holds that
\begin{align*}
    \sum_{\mathcal{T}=1}^{\mathcal{T}_{max}}(J^\mathcal{T}(\tilde{\bm{r}}^*_\varepsilon)-J^\mathcal{T}(\bm{r}^*)) \leq \langle \sum_{\mathcal{T}=1}^{\mathcal{T}_{max}}\nabla J^\mathcal{T}(\tilde{\bm{r}}^*_\varepsilon) , \bm{r}^*-\tilde{\bm{r}}^*_\varepsilon\rangle \leq \left\|\sum_{\mathcal{T}=1}^{\mathcal{T}_{max}}\nabla J^\mathcal{T}(\tilde{\bm{r}}^*_\varepsilon)\right\|_\infty \left\|\bm{r}^*-\tilde{\bm{r}}^*_\varepsilon \right\|_1\,.
\end{align*}
\textbf{Bounding the gradient norm.}
We show that there exists $k^*\in\argmin_{k\in[K]}\sum_{\mathcal{T}=1}^{\mathcal{T}_{max}}\partial_k J^\mathcal{T}(\tilde{\bm{r}}^*_{\varepsilon})$ such that
\begin{itemize}
    \item[(i)] $\left\|\sum_{\mathcal{T}=1}^{\mathcal{T}_{max}}\nabla J^\mathcal{T}(\tilde{\bm{r}}^*_\varepsilon)\right\|_\infty=\left|\sum_{\mathcal{T}=1}^{\mathcal{T}_{max}}\partial_{k^*} J^\mathcal{T}(\tilde{\bm{r}}^*_{\varepsilon})\right|$,
    \item[(ii)] $\forall k\in[K]:$ $\tilde{r}^*_{\varepsilon k}\leq \tilde{r}^*_{\varepsilon k^*}$.
\end{itemize}
(i) follows directly from the fact that $J^\mathcal{T}_k$ are monotonically decreasing functions, so all gradients are negative and the infinity norm is obtained by the smallest value.

(ii) follows from the K.K.T. conditions of the extreme point $\tilde{\bm{r}}^*_\varepsilon$, which read:
$\exists c \in \mathbb{R}$ such that
$\forall k\in[K]:$
\begin{align*}
&\sum_{\mathcal{T}=1}^{\mathcal{T}_{max}}\partial_k J^\mathcal{T}(\tilde{\bm{r}}^*_{\varepsilon}) = c\\
\text{or }& \sum_{\mathcal{T}=1}^{\mathcal{T}_{max}}\partial_k J^\mathcal{T}(\tilde{\bm{r}}^*_{\varepsilon}) \leq  c\text{ and }\tilde{r}^*_{\varepsilon k} = r_{max}\\  
\text{or }& \sum_{\mathcal{T}=1}^{\mathcal{T}_{max}}\partial_k J^\mathcal{T}(\tilde{\bm{r}}^*_{\varepsilon}) \geq  c\text{ and }\tilde{r}^*_{\varepsilon k} = r_{min}  \,.
\end{align*}
If the set $\{k\in [K]\,\mid \, \tilde{r}^*_{\varepsilon k} = r_{max}\}$ is not empty, then $k^*$ must lie in that set.
If $\forall k:\,\tilde{r}^*_{\varepsilon k}=r_{min}$, the statement is trivial.
Otherwise if there is no coordinate of value $r_{max}$ and at least one larger than $r_{min}$, we can simply choose  $k^*$ as $\argmax_{k\in[K]} \tilde{r}^*_{\varepsilon k}$ since the gradient is equal to $c$.

Note that $|\partial_k J^\mathcal{T}_k(r)|$ is monotonically decreasing in $r_k$ due to convexity and negativity of the gradients. Furthermore $\tilde{r}^*_{\varepsilon k^*} \geq \frac{B}{(1+\varepsilon)K}$ by definition.

This leads to the bound
\begin{align*}
    \left\|\sum_{\mathcal{T}=1}^{\mathcal{T}_{max}}\nabla J^\mathcal{T}(\tilde{\bm{r}}^*_\varepsilon)\right\|_\infty &= \left|\sum_{\mathcal{T}=1}^{\mathcal{T}_{max}}\partial_{k^*} J^\mathcal{T}(\tilde{\bm{r}}^*_{\varepsilon})\right| 
    &\leq \left|\sum_{\mathcal{T}=1}^{\mathcal{T}_{max}}\frac{\bar C_{k^*}\left(\frac{B}{(1+\varepsilon)K}\right)}{K}-\frac{(1+\varepsilon)K}{BK}\bar c_{k^*}\left(\frac{B}{(1+\varepsilon)K}\right)\right|\leq \frac{(1+\varepsilon)}{B}U\mathcal{T}_{max}\,.
\end{align*}
\textbf{Bounding the $\ell_1$-norm.}
From the K.K.T. conditions, we can directly infer that $\bm{r}^* \geq \tilde{\bm{r}}^*_{\varepsilon}$, or otherwise the K.K.T. conditions for the extreme point $\bm{r}^*$ are violated.
Therefore 
\begin{align*}
\lVert \bm{r}^*-\tilde{\bm{r}}^*_{\varepsilon}\rVert = \sum_{k=1}^K\bm{r}^*_k-\tilde{r}^*_{\varepsilon k} = B - \frac{B}{1+\varepsilon} = \frac{\varepsilon B}{1+\varepsilon}\,.
\end{align*}
Combining everything finishes the proof.
\end{proof}
\ \\
\ \\
\ \\
\emph{\textbf{Preliminaries for the proof of Theorem \ref{thm:in-policy regret}.}} For the proof, we require the following theorem and lemma
\begin{theorem}[\citet{}Bandit Algorithms. Theorem 28.4]
\label{thm:mirror descent}
Let $\eta > 0$ and $F$ be Legendre with domain $D$ and $\mathcal{K}_\varepsilon \subset \mathbb{R}^d$
be a nonempty convex set with $\operatorname{int}(\operatorname{dom}(F)) \cap \mathcal{K}_\varepsilon \neq \emptyset$. Let
$\bm{r}^1, \dots , \bm{r}^{\mathcal{T}+1}$ be the actions chosen by mirror descent, which are assumed to exist.
Furthermore assume that for any $\mathcal{T}\in[\mathcal{T}_{max}]$: $\nabla F(\bm{r}^\mathcal{T})-\eta \bm{g}^\mathcal{T} \in \operatorname{int}(\operatorname{dom}(F^*))$, then for 
\begin{align*}
\tilde{\bm{r}}^{\mathcal{T}+1} \triangleq \argmin_{\bm r\in D} \eta\langle \bm r,\bm{g}^\mathcal{T}\rangle + D_F(\bm r,\bm{r}^\mathcal{T})\,, 
\end{align*}
the regret of mirror descent is bounded for any $\bm r \in \mathcal{K}_\varepsilon$ by
\begin{align*}
    \sum_{\mathcal{T}=1}^{\mathcal{T}_{max}}\langle \bm{r}^\mathcal{T}-\bm r,\bm{g}^\mathcal{T}\rangle \leq \eta^{-1}\left(F(\bm r)-F(\bm{r}^1) +\sum_{\mathcal{T}=1}^{\mathcal{T}_{max}}D_F(\bm{r}^\mathcal{T},\tilde{\bm{r}}^{\mathcal{T}+1})\right)\,.
\end{align*}
\end{theorem}

\begin{lemma}
\label{lem:log bound}
 For any $x\geq-\frac{1}{2}:\,    -\log(1+x) +x \leq x^2$.
\end{lemma}
\begin{proof}
For $x\in (-1/2,0)$ the gradient of the LHS is larger than the RHS, while  for $x>0$ it reverses.
That means $x=0$ is a maximum of $-\log(1+x)+x-x^2$ for $x\in[-1/2,\infty)$. Therefore
\[
\forall x\geq -1/2:\,-\log(1+x)+x-x^2\leq -\log(1+0)+0-0^2=0\,,
\]
which concludes the proof.
\end{proof}
\ \\
\ \\
\ \\
\emph{\textbf{Theorem \ref{thm:in-policy regret}.} 
For any sequence of convex functions $(J^{\mathcal{T}})_{\mathcal{T}=1}^{\mathcal{T}_{max}}$ and learning rate $0<\eta<\frac{K\varepsilon}{2U}$, the in-policy regret of \algo\ is bounded by
\begin{align*}
    \Reg(\bm{r}^*_\varepsilon) \leq \frac{K}{\eta}\log\left(\frac{B}{r_{min}K}\right)+ \eta\frac{U^2}{\varepsilon K}\mathcal{T}_{max}\,.
\end{align*}
}

\begin{proof}
Since the functions $J^{\mathcal{T}}$ are convex, we have 
\begin{align*}
    \E\left[\sum_{\mathcal{T}=1}^{\mathcal{T}_{max}}(J^\mathcal{T}(\bm{r}^\mathcal{T})-J^\mathcal{T}(\bm{r}^*_\varepsilon))\right]\leq \E\left[\sum_{\mathcal{T}=1}^{\mathcal{T}_{max}}\langle \bm{r}^\mathcal{T}-\bm{r}^*_\varepsilon,\nabla J^\mathcal{T}(\bm{r}^\mathcal{T})\rangle\right].
\end{align*}
Furthermore the loss estimators are conditionally independent and unbiased, so
\begin{align*}
    \E\left[\sum_{\mathcal{T}=1}^{\mathcal{T}_{max}}\langle \bm{r}^\mathcal{T}-\bm{r}^*_\varepsilon,\nabla J^\mathcal{T}(\bm{r}^\mathcal{T})\rangle\right]=\E\left[\sum_{\mathcal{T}=1}^{\mathcal{T}_{max}}\langle \bm{r}^\mathcal{T}-\bm{r}^*_\varepsilon,\bm{g}^\mathcal{T}\rangle\right]\,.
\end{align*}
We verify that we can apply Theorem~\ref{thm:mirror descent}.
Recall our potential is $F(\bm r) = -\sum_{k=1}^K\log(r_k)$ with domain $D=(0,\infty)^K$.
The convex conjugate is $F^*(\bm{y}) = -K-\sum_{k=1}^K \log(-y_k)$ with interior domain $(-\infty,0)^K$.
It holds
\begin{align*}
    \partial_kF(\bm{r}^\mathcal{T})-\eta g_k^\mathcal{T} = -\frac{1}{r^\mathcal{T}_k}-\eta g_k^\mathcal{T} \leq -\frac{1}{r^\mathcal{T}_k}+\frac{\eta U}{r^\mathcal{T}_kK \varepsilon} < 0\,,
\end{align*}
which completes the requirements. By Theorem~\ref{thm:mirror descent}
\begin{align*}
    \E\left[\sum_{\mathcal{T}=1}^{\mathcal{T}_{max}}\langle \bm{r}^\mathcal{T}-\bm{r}^*_\varepsilon,\bm{g}^\mathcal{T}\rangle\right] \leq \eta^{-1}\left(F(\bm r)-F(\bm{r}^1) +\sum_{\mathcal{T}=1}^{\mathcal{T}_{max}}\E\left[D_F(\bm{r}^\mathcal{T},\tilde{\bm{r}}^{\mathcal{T}+1})\right]\right)\,.
\end{align*}
Now we bound the Bregman divergence terms.
\begin{align*}
    &\tilde{r}^{\mathcal{T}+1}_k = \argmin_{r\in (0,\infty)} \eta rg_k^\mathcal{T}  -\log(\frac{r}{r_k^\mathcal{T}})+\frac{r}{r_k^\mathcal{T}}-1\\
    &\eta g^\mathcal{T}_k -\frac{1}{\tilde{r}^{\mathcal{T}+1}_k} +\frac{1}{r_k^\mathcal{T}} = 0\\
    &\tilde{r}^{\mathcal{T}+1}_k = \frac{r_k^{\mathcal{T}}}{1+\eta g^\mathcal{T}_kr_k^{\mathcal{T}}}
\end{align*}

Denote $\mathbb{B}^\mathcal{T}_k\sim Ber(\varepsilon)$ the indicator of having used the exploration for estimating the gradient $\partial_kJ^\mathcal{T}$ in round $\mathcal{T}$. by the definition of the gradient estimator, it is bounded by $|g^\mathcal{T}_k|\leq \mathbb{B}^\mathcal{T}_k\frac{2U}{\varepsilon K}+(1-\mathbb{B}^\mathcal{T}_k)\frac{U}{K}$.
Hence
\begin{align*}
    D_F(\bm{r}^\mathcal{T},\tilde{\bm{r}}^{\mathcal{T}+1}) &= \sum_{k=1}^K\left(-\log\left(\frac{r_k^\mathcal{T}}{\tilde{r}^{\mathcal{T}+1}_k}\right) + \frac{r_k^\mathcal{T}}{\tilde{r}^{\mathcal{T}+1}_k}-1\right)\\
    & = \sum_{k=1}^K\left(-\log(1+\eta g^\mathcal{T}_kr_k^{\mathcal{T}})+\eta g^\mathcal{T}_kr_k^{\mathcal{T}}\right)
    \overset{\mbox{(Lem.~\ref{lem:log bound})}}{\leq} \sum_{k=1}^K\left(\eta g^\mathcal{T}_kr_k^{\mathcal{T}})\right)^2\\
    &\leq \eta^2\sum_{k=1}^K\mathbb{B}^\mathcal{T}_k\frac{U^2}{\varepsilon^2K^2}\,.
\end{align*}
Since $\mathbb{B}^\mathcal{T}_k\sim Ber(\varepsilon)$, we have
\begin{align*}
    \mathbb{E}\left[D_F(\bm{r}^\mathcal{T},\tilde{\bm{r}}^{\mathcal{T}+1}) \right] 
    &\leq \frac{\eta^2U^2}{\varepsilon K}\,.
\end{align*}
Combining everything completes the proof
\begin{align*}
    \E\left[\sum_{\mathcal{T}=1}^{\mathcal{T}_{max}}\langle \bm{r}^\mathcal{T}-\bm{r}^*_\varepsilon,\bm{g}^\mathcal{T}\rangle\right] 
    &\leq \sum_{k=1}^K\frac{-\log(r^*_{\varepsilon k})+\log(r^1_k)}{\eta} + \eta\frac{U^2}{\varepsilon K}\mathcal{T}_{max}\\
    &\leq \frac{K}{\eta}\log\left(\frac{B}{r_{min}K}\right)+ \eta\frac{U^2}{\varepsilon K}\mathcal{T}_{max}\,.
\end{align*}
\end{proof}

\section{Details of Empirical Evaluation \label{sec:emp_details}}

\textbf{Constructing polynomial cost-generating processes.} The cost-generating processes $c_k(\tau) = a_k  \tau^{p_k}$ used in the experiments in Figures \ref{fig:ms_vs_ams} and \ref{fig:sgd_vs_ams} were constructed and operated as follows: 
\begin{itemize}[noitemsep,topsep=0pt]
\item During a sync-mode play of arm $k$, a function $c_k(.)$ is randomly chosen until the next sync-mode play by sampling $a_k \sim  Uniform([\overline{a}_k - \overline{a}_k \cdot noise, \overline{a}_k + \overline{a}_k \cdot noise])$, where $noise \in [0,1]$ is a parameter and $\overline{a}_k$ is both problem-instance- and arm-specific. To construct a problem instance, our generator randomly chose $\overline{a}_k \sim Uniform[0,1]$ for each $k$ at problem creation time. The $noise$ parameter in our experiments was shared by all problem instances and all arms, its value as in Table \ref{t:fixed_prob_params}.
\item $p_k$ is a parameter chosen by our problem generator at the instance creation time from the sigmoid prior $\sigma(scale \cdot Uniform[0,1])$, where the $scaling$ parameter in our experiments is shared by all problem instances and all arms, with its value as in Table \ref{t:fixed_prob_params}. Thus, $p_k \in (0, 1)$ but $p_k$ values $> 0.5$ are more likely.
\end{itemize}

\begin{table}[h]
\centering
\begin{tabular}{c c}
\toprule
\textbf{Problem parameter} & \textbf{Value}  \\
\midrule
$r_{min}$ & $0.025$\\
$r_{max}$ & $3$\\
$K$ & $100$ \\
$B$ & $0.4 \cdot K$\\
$noise$ & $0.1$\\
$scaling$ & $5$ \\
\bottomrule\\
\end{tabular}
\caption{Problem generator parameters for results in Figures \ref{fig:ms_vs_ams} and \ref{fig:sgd_vs_ams}.}
\label{t:fixed_prob_params}
\end{table}

\textbf{Hyperparameter tuning.} We performed a grid search with a fixed random number generator seed, 0, on problems generated using parameters in Table \ref{t:fixed_prob_params}. Namely, for each considered parameter combination, we ran each algorithm \emph{once} for a fixed number of update rounds equal to $horizon \cdot (1/r_{min}) / l_{upd\_round}$, where $horizon = 240$ on the problem setup described in Section \ref{sec:exps}, and generated curves similar to those in Figures \ref{fig:ms_vs_ams} and \ref{fig:sgd_vs_ams}. For \algo, $l_{upd\_round}$ was always 1. For \algoprac\ and \algopracpsgd, $l_{upd\_round} < 8$ consistently caused both to either be very unstable or diverge (independently of $\eta$), while $l_{upd\_round} > 40$ would cause them to update arm play rates less frequently than \algo, so we didn't consider $l_{upd\_round} \notin [8,40]$ during grid search. We inspected the resulting curves and for each algorithm chose a parameter combination whose curve showed convergence to the lowest cost $J$ at the end but remained stable (showed little jitter). Please refer to the implementation at \href{https://github.com/microsoft/Optimal-Freshness-Crawl-Scheduling}{\color{blue}{https://github.com/microsoft/Optimal-Freshness-Crawl-Scheduling}} for specific parameter combinations over which we performed the grid search.   

The chosen parameter combinations that we used to get the results in Figures \ref{fig:ms_vs_ams} and \ref{fig:sgd_vs_ams} are:
\begin{align*}
    \eta &= 2.7 \mbox{ for \algo} \\
    \eta &= 1.6,  l_{upd\_round} = 20 \mbox{ for \algoprac}\\
    \eta &= 0.08,  l_{upd\_round} = 20 \mbox{ for \algopracpsgd}
\end{align*}


\section{Additional Results \label{sec:emp_add_res}}

Besides the polynomial cost-generating processes in Section \ref{sec:exps}, we experimented with Poisson process-based costs of the kind common in the Web crawling literature \cite{cho-tds03,azar-pnas18,kolobov-neurips19,kolobov-sigir19}. Here, each arm's cost-generating process $c_k(\tau)$ is driven by an associated Poisson point process $pois(\lambda_k)$, whose rate $\lambda_k$ is sampled from $Uniform[0.005, 5.0]$ at problem generation time. For each arm, we let $c_k(\tau) \triangleq 1$ if $pois(\lambda_k)$ has generated at least one event in time $\tau$ since the latest play of arm $k$, and 0 otherwise. In Web crawling, $pois(\lambda_k)$ models the number of meaningful changes on a Web page since it was last crawled. For problem instances in this experiment, we used problem generation parameters in Table \ref{t:fixed_prob_params_poisson}.

\begin{table}[h]
\centering
\begin{tabular}{c c}
\toprule
\textbf{Problem parameter} & \textbf{Value}  \\
\midrule
$r_{min}$ & $0.025$\\
$r_{max}$ & $6$\\
$K$ & $100$ \\
$B$ & $0.4 \cdot K$\\
\bottomrule\\
\end{tabular}
\caption{Problem generator parameters for results in Figures \ref{fig:ms_vs_ams_bp} and \ref{fig:sgd_vs_ams_bp}.}
\label{t:fixed_prob_params_poisson}
\end{table}

The algorithms' hyperparameter values were chosen using the same procedure as described in Appendix \ref{sec:emp_details}, in combination with $\varepsilon=0.05$:

\begin{align*}
    \eta &= 5 \mbox{ for \algo} \\
    \eta &= 1.3,  l_{upd\_round} = 8 \mbox{ for \algoprac}\\
    \eta &= 0.5,  l_{upd\_round} = 40 \mbox{ for \algopracpsgd}
\end{align*}

The experiment's results are shown in Figures \ref{fig:ms_vs_ams_bp} and \ref{fig:sgd_vs_ams_bp}. Qualitatively, they resemble those for polynomial cost-generating processes in Figures \ref{fig:ms_vs_ams} and \ref{fig:sgd_vs_ams} in Section \ref{sec:exps}: \algoprac's empirical converge rate approximately matches that of \algo\ (which, however, makes an unrealistic assumption about free arm pulls at the start of each round) and exceeds that of \algopracpsgd, although the advantage of \algoprac\ over \algopracpsgd\ is less pronounced than in the first experiment.

\begin{figure*}
\begin{minipage}[t]{0.48\textwidth}
 \vspace{4pt}%

 \vspace{-0.20in}
 \captionof{figure}{\small \algo's and \algoprac's convergence on the problem with Poisson process-based cost-generating processes. Like in Figure \ref{fig:ms_vs_ams}, the two algorithms exhibit very similar convergence behavior, despite \algoprac\ not relying on free arm resets as \algo\ does.}  
 \label{fig:ms_vs_ams_bp}
 \end{minipage}
 \hspace{0.15in}
\begin{minipage}[t]{0.48\textwidth}
 \vspace{4pt}%
 %
 \vspace{-0.2in}
 \captionof{figure}{\small \algoprac\ vs. \algopracpsgd. As in the experiment with polynomial cost-generating processes, \algoprac\ arrives at good policies faster than \algopracpsgd, but the former's advantage is less pronounced than in Figure \ref{fig:sgd_vs_ams}.}
 \label{fig:sgd_vs_ams_bp}
 \end{minipage}
 \end{figure*}

\end{document}